\documentclass[a4paper]{article}

\usepackage{graphicx} 
\usepackage{subfigure}

\usepackage{bm}
\usepackage{dsfont}

\usepackage{natbib}

\usepackage{algorithm}
\usepackage[noend]{algpseudocode}
\algrenewcommand{\algorithmiccomment}[1]{$\vartriangleright$ #1}
\algrenewcommand{\algorithmicreturn}{\textbf{Return: }}
\algnewcommand\algorithmicinput{\textbf{Input: }}
\algnewcommand\Input{\State \algorithmicinput}

\usepackage[strict]{changepage}

\usepackage{amsmath}
\usepackage{amssymb}

\usepackage{amsthm}
\usepackage{mathtools}
\mathtoolsset{showonlyrefs}

\usepackage{color}
\usepackage{xcolor}
\usepackage{hyperref}
\ifdefined\nohyperref\else\ifdefined\hypersetup
  \definecolor{mydarkblue}{rgb}{0,0.08,0.45}
  \hypersetup{ %
    pdftitle={},
    pdfauthor={},
    pdfsubject={},
    pdfkeywords={},
    pdfborder=0 0 0,
    pdfpagemode=UseNone,
    colorlinks=true,
    linkcolor=mydarkblue,
    citecolor=mydarkblue,
    filecolor=mydarkblue,
    urlcolor=mydarkblue,
    pdfview=FitH}

  \ifdefined\isaccepted \else
    \hypersetup{pdfauthor={Anonymous Submission}}
  \fi
\fi\fi

\usepackage[mathscr]{euscript}

\usepackage{enumitem}

\usepackage{nicefrac}

\usepackage{booktabs}

\usepackage[super]{nth}

\colorlet{darkblue}{blue!80!black}
\colorlet{darkred}{red!80!black}
\colorlet{darkgreen}{green!60!black}
\colorlet{darkmagenta}{orange!80!black}
\colorlet{darkyellow}{purple!80!black}

\newcommand{\myred}[1]{\textcolor{darkred}{#1}}
\newcommand{\myblue}[1]{\textcolor{darkblue}{#1}}
\newcommand{\mygreen}[1]{\textcolor{darkgreen}{#1}}
\newcommand{\mygray}[1]{\textcolor{gray}{#1}}
\newcommand{\mymagenta}[1]{\textcolor{darkmagenta}{#1}}
\newcommand{\myyellow}[1]{\textcolor{darkyellow}{#1}}

\newtheorem{lemma}{Lemma} 
\newtheorem{proposition}{Proposition}

\newtheorem{definition}{Definition}

\def\EE{{\mathbb E}}

\def\RR{{\mathbb R}}
\def\PP{{\mathbb P}}

\def\x{{\bm x}}

\def\e{{\bm e}}
\def\c{{\bm c}}

\def\v{{\bm v}}
\def\w{{\bm w}}

\def\c{{\bm c}}

\def\l{{\bm l}}

\def\A{{\bm A}}
\def\Q{{\bm Q}}
\def\L{{\bm L}}
\def\softmax{{\textrm{softmax}}}

\def\B{{\bm B}}
\def\C{{\bm C}}

\def\Ytrue{{{\bm Y}_{\textrm{true}}}}

\def\T{{\bm T}}
\def\E{{\bm E}}

\def\J{{\bm J}}
\def\L{{\bm L}}
\def\M{{\bm M}}

\def\U{{\bm U}}

\def\X{{\bm X}}
\def\Y{{\bm Y}}
\def\Z{{\bm Z}}
\def\H{{\bm H}}
\def\p{{\bm \p}}

\def\b{{\bm b}}
\def\a{{\bm a}}

\def\q{{\bm q}}
\def\s{{\bm s}}
\def\u{{\bm u}}

\def\x{{\bm x}}
\def\W{{\bm W}}
\def\Z{{\bm Z}}
\def\z{{\bm z}}

\def\y{{\bm y}}

\def\p{{\bm p}}
\def\q{{\bm q}}

\newcommand{\Simplex}{\triangle}

\newcommand{\bTheta}{\bm \Theta}

\newcommand{\blambda}{\boldsymbol{\lambda}}

\renewcommand\P{{\bm P}}

\newcommand{\cP}{{\mathcal P}}

\newcommand{\cE}{{\mathcal E}}
\newcommand{\cV}{{\mathcal V}}
\newcommand{\cD}{{\mathcal D}}

\newcommand{\cC}{{\mathcal C}}

\newcommand{\cX}{{\mathcal X}}
\newcommand{\cY}{{\mathcal Y}}

\newcommand{\btheta}{\boldsymbol{\theta}}

\def\bigO{{\mathcal{O}}}

\DeclareMathOperator*{\argmin}{argmin}
\DeclareMathOperator*{\argmax}{argmax}

\usepackage{xcolor}
\definecolor{dred}{rgb}{0.8,0,0}
\definecolor{dgreen}{rgb}{0,0.8,0}
\definecolor{dblue}{rgb}{0,0,0.8}
\definecolor{dpurple}{rgb}{0.8,0,0.8}

\newcommand\wrt{w.r.t.\ }

\DeclareMathOperator*{\conv}{conv}
\DeclareMathOperator*{\supp}{supp}
\newcommand\maxOmega{\textstyle{\max_\Omega}}
\newcommand\minOmega{\textstyle{\min_\Omega}}

\newcommand\dtwOmega{\text{DTW}_\Omega}
\newcommand\viterbiOmega{\text{Vit}_\Omega}

\newcommand\dpz{\text{DP}}
\newcommand\dpOmega{\text{DP}_\Omega}
\newcommand\lp{\text{LP}}
\newcommand\lpOmega{\text{LP}_\Omega}

\newcommand\zeros{\mathbf{0}}
\newcommand\ones{\mathbf{1}}

\newcommand\diag{\text{Diag}}

\newcommand\partialfrac[2]{\frac{\partial #1}{\partial #2}}

\DeclareMathOperator*{\dom}{dom}

\newcommand\proba[2]{p_{#2}(#1)}
\newcommand\expect[2]{\EE_{#2}[#1]}

\newcommand*{\ie}{\textit{i.\,e.}}

\newlength{\offsetpage}
\newenvironment{widepage}{\begin{adjustwidth}{-\offsetpage}{-\offsetpage}%
    \addtolength{\textwidth}{2\offsetpage}}%
{\end{adjustwidth}}

\newenvironment{keyword}{%
\vspace{.1em}\begin{center}\begin{minipage}{0.89\textwidth}
\textbf{Index words.}}
{\par\noindent\end{minipage}\end{center}}

\usepackage[top=1in, bottom=0.9in, left=0.9in, right=0.9in]{geometry}

\usepackage{fancyhdr}

\begin{document}

\makeatletter
\def\blfootnote{\gdef\@thefnmark{}\@footnotetext}
\makeatother

\title{Differentiable Dynamic Programming for \\ Structured Prediction and Attention}

\author{
\vspace{.5em}
 \hfill
  \begin{tabular}[t]{c}
      Arthur Mensch\textsuperscript{*}\\
  Inria Parietal\\
  Saclay, France\\
  \texttt{arthur.mensch@m4x.org} \\
  \\
  Mathieu Blondel\\
  NTT Communication Science Laboratories\\
  Kyoto, Japan \\
  \texttt{mathieu@mblondel.org} \\
  \end{tabular}
  \vspace{.5em}
}
\maketitle

\begin{abstract}
Dynamic programming (DP) solves a variety of structured combinatorial
problems
by iteratively breaking them down into smaller subproblems. In spite of their versatility, DP
algorithms are usually non-differentiable, which hampers their use as a layer in
neural networks trained by backpropagation.  To address this issue, we
propose to smooth the max operator in the dynamic programming recursion, using a
strongly convex regularizer. This allows to relax both the optimal value and
solution of the original combinatorial problem, and turns a broad
class of DP algorithms into differentiable operators.  Theoretically, we provide a
new probabilistic perspective on backpropagating through these DP operators, and
relate them to inference in graphical models.  We derive two particular
instantiations of our framework, a smoothed Viterbi algorithm for sequence
prediction and a smoothed DTW algorithm for time-series alignment. We
showcase these instantiations on two structured prediction tasks and on
structured and sparse attention for neural machine translation.
\end{abstract}


\blfootnote{\textsuperscript{*}Work performed during an internship at NTT
Communication Science laboratories, Kyoto, Japan.}

\section{Introduction}
\label{sec:introduction}

Modern neural networks are composed of multiple layers of nested functions.  
Although layers are usually constituted
of elementary linear algebraic operations and simple non-linearities,
there is a growing need for layers that output the \textit{value} or the
\textit{solution} of an optimization problem. This can be used to design loss
functions that capture relevant regularities in the input
\citep{lample_2016,soft_dtw} or to
create layers that impose prior structure on the output
\citep{kim_structured_2017,optnet,niculae_blondel_2017,
djolonga_2017}.

Among these works, several involve a convex optimization problem
\citep{optnet,niculae_blondel_2017,djolonga_2017}; others solve certain combinatorial
optimization problems by dynamic programming
\citep{lample_2016,kim_structured_2017,soft_dtw}.  However, because 
dynamic programs~\citep{bellman_1952} are usually non-differentiable, virtually all these works
resort to the formalism of conditional random fields (CRFs) \citep{lafferty_crf}, which
can be seen as changing the semiring used by the dynamic program ---
replacing all values by their exponentials and all $(\max,+)$ operations with
$(+,\times)$ operations \citep{verdu1987}. While this modification smoothes the
dynamic program, it looses the sparsity of solutions, since hard assignments
become soft ones. Moreover, a general understanding of how to relax and
differentiate dynamic programs is lacking. In this work, we propose to do so by
leveraging smoothing~\citep{moreau_proximite_1965, nesterov_smooth} and
backpropagation~\citep{linnainmaa_1970}. We make the
following contributions.

1) We present a unified framework for
turning a broad class of dynamic programs (DP) into differentiable operators.
Unlike existing works, we propose to change the semiring to
use $(\maxOmega, +)$ operations, where $\maxOmega$ is a max
operator smoothed with a strongly convex regularizer $\Omega$
(\S\ref{sec:smoothed_max_op}). 

2) We show that the resulting DP operators, that we call $\dpOmega$, are
smoothed relaxations of the original DP algorithm and satisfy several key
properties, chief among them convexity.  In addition, we show that their
gradient, $\nabla \dpOmega$, is equal to the \textit{expected trajectory} of a
certain random walk and can be used as a sound relaxation to the original
dynamic program's solution. Using negative entropy for $\Omega$ recovers
existing CRF-based works from a different perspective --- we provide new
arguments as to why this $\Omega$ is a good choice. On the other hand, using squared
$\ell_2$ norm for $\Omega$ leads to new algorithms whose expected solution is
\textit{sparse}. We derive a clean and efficient method to backpropagate
gradients, both through $\dpOmega$ and $\nabla \dpOmega$.  This allows us to
define differentiable DP layers that can be incorporated in neural networks
trained end-to-end (\S\ref{sec:differentiable_dp_layers}).  

3) We illustrate how to to derive two particular instantiations of our
framework, a smoothed Viterbi algorithm for sequence prediction and a smoothed
DTW algorithm for supervised time-series alignment (\S\ref{sec:examples}).
The latter is illustrated in Figure \ref{fig:smooth_dtw}.
Finally, we showcase these two instantiations on structured prediction
tasks (\S\ref{sec:structured_prediction}) and on structured attention
for neural machine translation (\S\ref{sec:structured_attention}).

\paragraph{Notation.} 

We denote scalars, vectors and matrices using lower-case, bold lower-case and
bold upper-case letters, \textit{e.g.,} $y$, $\y$ and $\Y$.  We denote the
elements of $\Y$ by $y_{i,j}$ and its rows by $\y_i$.
We denote the Frobenius inner product between
$\A$ and $\B$ by $\langle \A, \B \rangle \triangleq \sum_{i,j} a_{i,j}
b_{i,j}$. We denote the $(D-1)$-probability simplex by $\Simplex^D
\triangleq \{\blambda \in \RR_+^D \colon \|\blambda\|_1 = 1\}$.
We write $\conv(\cY) 
\triangleq \{\sum_{\Y \in \cY}
\lambda_{\Y} \Y \colon \blambda \in \Simplex^{|\cY|}\}$
the convex hull of $\cY$,
$[N]$ the set $\{1,\dots,N\}$ and
$\supp(\x) \triangleq \{j \in [D] \colon
x_j \neq 0\}$ the support of $\x \in \RR^D$.
We denote the Shannon entropy by $H(\q)
\triangleq \sum_i q_i \log q_i$.

We will release an optimized modular \textit{PyTorch} implementation
for reproduction and reuse.

\begin{figure}[t]
    \centering
    \includegraphics[width=.8\linewidth]{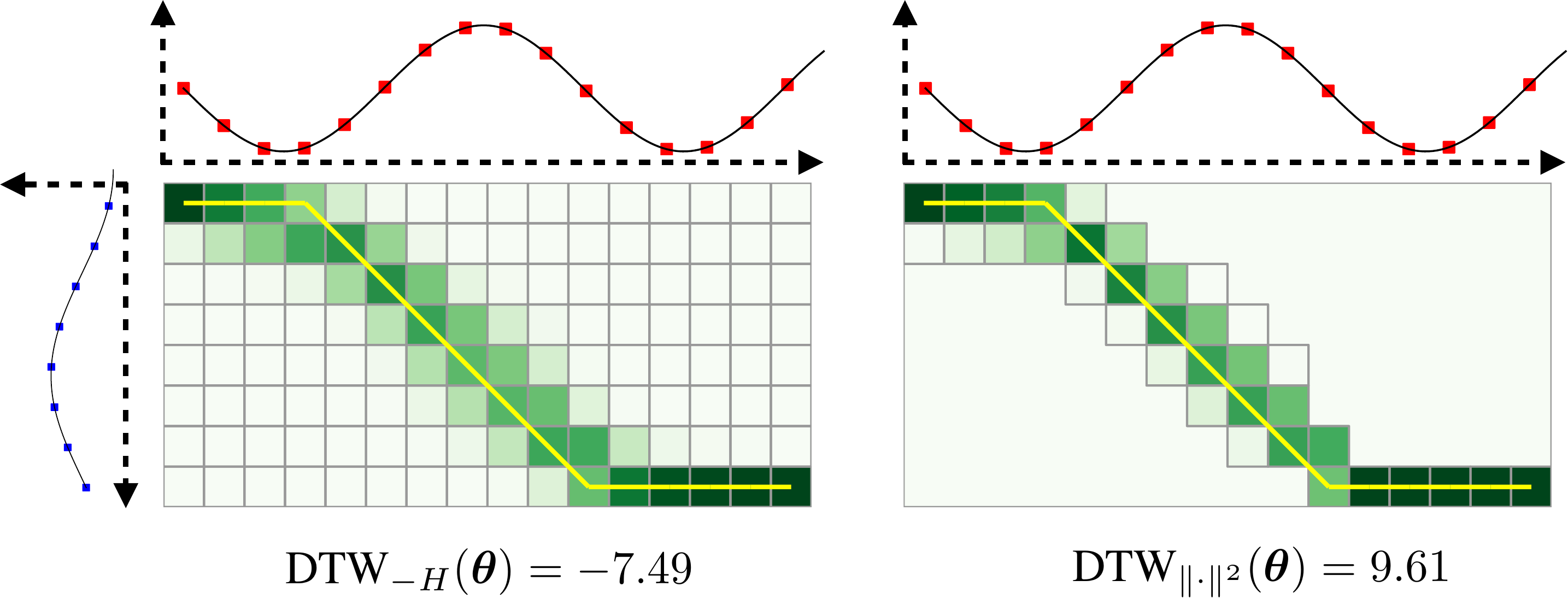}
    \caption{$\dtwOmega(\btheta)$ is an instantiation of the proposed smoothed
        dynamic programming operator, $\dpOmega(\btheta)$, to the dynamic time warping (DTW)
        computational graph.  In this picture, $\btheta$ is the squared
        Euclidean distance matrix between the observations of two time-series.  The gradient $\nabla
        \dtwOmega(\btheta)$ is equal to the expected alignment under a certain
        random walk characterized in \S\ref{sec:differentiation} and is a
        sound continuous relaxation to the hard DTW alignment between the two
        time-series (here depicted with a yellow path). Unlike negentropy
        regularization (left), $\ell_2^2$ regularization leads to exactly
        sparse alignments (right).
Our framework allows to backpropagate through
        both $\dtwOmega(\btheta)$ and $\nabla \dtwOmega(\btheta)$, which makes
        it possible to learn the distance matrix~$\btheta$ end-to-end.  }
    \label{fig:smooth_dtw}
\end{figure}

\section{Smoothed max operators}
\label{sec:smoothed_max_op}

In this section, we introduce smoothed max operators \citep{nesterov_smooth,
beck_smoothing_2012,niculae_blondel_2017},
that will serve as a powerful and generic abstraction to define differentiable
dynamic programs in \S\ref{sec:differentiable_dp_layers}.
Formally, let $\Omega: \RR^D \to \RR$ be a strongly convex regularizer
on $\Simplex^D$ and let $\x \in \RR^D$. We define the max operator smoothed by
$\Omega$ as:
\begin{equation}
\maxOmega(\x) \triangleq \displaystyle{\max_{\q \in \Simplex^D}} ~
\langle \q, \x \rangle  - \Omega(\q).
\label{eq:smoothed_max}
\end{equation}
In other words, $\maxOmega$ is the convex conjugate of $\Omega$, restricted to
the simplex.  From the duality between strong convexity and smoothness,
$\maxOmega$ is smooth: differentiable everywhere and with Lipschitz continuous
gradient.  Since the argument that achieves the maximum in
\eqref{eq:smoothed_max} is unique, from Danskin's theorem
(\citeyear{danskin_theorem}),
it is equal to the gradient of $\maxOmega$:
\begin{equation}
\nabla \maxOmega(\x) = \displaystyle{\argmax_{\q \in \Simplex^D}} ~ \langle \q, \x \rangle -
\Omega(\q).
\end{equation}
The gradient is differentiable almost
everywhere for any strongly-convex $\Omega$ (everywhere for negentropy). 
Next, we state properties that will be useful
throughout this paper.
\begin{lemma}{Properties of $\maxOmega$ operators}

Let $\x = (x_1, \dots, x_D)^\top \in \RR^D$.
\begin{enumerate}[topsep=0pt,itemsep=3pt,parsep=3pt]
    \item \label{property:maxOmega_bound} Boundedness: If $\Omega$ is
        lower-bounded by $L_{\Omega,D}$ and upper-bounded by $U_{\Omega,D}$ on
        the simplex $\Simplex^D$, then \\
        $\max(\x) - U_{\Omega,D} \le \maxOmega(\x) \le \max(\x) - L_{\Omega,D}$.
    \item \label{property:maxOmega_distrib} Distributivity of $+$ over
        $\maxOmega$: $\maxOmega(\x + c \ones) = \maxOmega(\x) + c \quad
        \forall c \in \RR$.
    \item \label{property:maxOmega_comm} Commutativity: If $\Omega(\P \q) =
        \Omega(\q)$, where $\P$ is a permutation matrix, then $\maxOmega(\P \x)
        = \maxOmega(\x)$.
    \item \label{property:maxOmega_non_decreasing} Non-decreasingness in each
        coordinate: 
        $\maxOmega(\x) \le \maxOmega(\y) \quad \forall \x \le \y$.
    \item \label{property:nabla_maxOmega_j_equal_0}
        Insensitivity to $-\infty$: 
        $x_j = -\infty \Rightarrow \nabla \maxOmega(\x)_j = 0$.
\end{enumerate}
\label{lemma:max_Omega_properties}
\end{lemma}
Proofs are given in \S\ref{appendix:proof_lemma_max_Omega_properties}.
In particular, property~\ref{property:maxOmega_comm} holds whenever $\Omega(\q)
= \sum_{i=1}^D \omega(q_i)$, for some function $\omega$.
We focus in this paper on two specific regularizers $\Omega$: the negentropy $-H$ and
the squared $\ell_2$ norm. For these choices, all properties above are satisfied and
we can derive closed-form expressions for $\maxOmega$, its gradient and
its Hessian --- see \S\ref{appendix:omega_examples}.  When using
negentropy, $\maxOmega$ becomes the \textit{log-sum-exp} and $\nabla \maxOmega$ the
\textit{softmax}. The former satisfies
associativity, which as we shall see, makes it natural to use in dynamic
programming.  With the squared $\ell_2$ regularization, as observed
by~\citet{sparsemax,niculae_blondel_2017}, the gradient $\nabla \maxOmega$ is \textit{sparse}. This
will prove useful to enforce sparsity in the models we study.

\section{Differentiable DP layers}
\label{sec:differentiable_dp_layers}

Dynamic programming (DP) is a generic way of solving combinatorial optimization
problems by recursively solving problems on smaller sets.
We first introduce this category of algorithms in a broad setting, then
use smoothed max operators to define differentiable DP layers.

\subsection{Dynamic programming on a DAG}
\label{sec:dp_on_dag}

Every problem solved by dynamic programming reduces to finding
the highest-scoring path between a start node and an end node,
 on a weighted directed acyclic graph (DAG). We therefore introduce our formalism
 on this generic problem, and give concrete examples in \S\ref{sec:examples}.

Formally, let $G=(\cV, \cE)$ be a DAG, with nodes~$\cV$ and
edges $\cE$. We write $N = |\cV| \ge 2$ the number of nodes.
Without loss of generality, we number the nodes in topological order, from~$1$
(start) to $N$ (end), and thus $\cV = [N]$. 
Node~$1$ is the only node without parents, and node~$N$ the only node
without children. 
Every directed edge $(i, j)$
from a parent node $j$ to a child node $i$ has a weight $\theta_{i,j} \in
\RR$. We gather the edge weights in a matrix $\btheta \in \bTheta \subseteq
\RR^{N \times N}$, setting $\theta_{i,j} = - \infty$ if $(i, j) \notin \cE$ and $\theta_{1, 1} = 1$.
We consider the set $\cY$ of all paths in $G$ from node $1$ to node~$N$.
Any path $\Y \in \cY$ can
be represented as a $N \times N$ binary matrix, with $y_{i,j} = 1$ if the path 
goes through the edge $(i, j)$ and $y_{i,j} = 0$ otherwise.
In the sequel, paths will have a one-to-one correspondence with discrete
structures such as sequences or alignments.
Using this representation, $\langle \Y, \btheta \rangle$ corresponds to the
cumulated sum of edge weights, along the path $\Y$. The computation
of the \textit{highest score} among all
paths amounts to solving the combinatorial problem
\begin{equation}
\lp(\btheta) \triangleq \max_{\Y \in \cY} ~ \langle \Y, \btheta \rangle \in
\mathbb{R}.
\label{eq:LP}
\end{equation}

Although the size of $\cY$ is in general exponential in $N$, $\lp(\btheta)$ can be
computed in one topologically-ordered pass over $G$ using \textit{dynamic programming}.
We let $\cP_i$ be the set of parent nodes of node $i$ in
graph $G$ and define recursively
\begin{align}
    v_1(\btheta) &\triangleq 0 \\
    \forall\,i \in [2, \dots, N]: ~
    v_i(\btheta) &\triangleq \max_{j \in \cP_i}
    \theta_{i,j} + v_j(\btheta).
\label{eq:general_recursion}
\end{align}
This algorithm outputs
$\dpz(\btheta) \triangleq v_N(\btheta)$. We now show that this is precisely the
highest score among all paths. 

\begin{proposition}{Optimality of dynamic programming}%

\normalfont
\centering
$\forall \btheta \in \bTheta: \quad    \dpz(\btheta) = \lp(\btheta)$
\label{proposition:recursion}
\end{proposition}
The optimality of the recursion~\eqref{eq:general_recursion} is a
well-known result~\citep{bellman_1952}.
We prove it again with our formalism in
\S\ref{appendix:proof_proposition_recursion}, since it exhibits the two
key properties that the max operator must satisfy to guarantee optimality:
\textit{distributivity} of $+$ over it and \textit{associativity}.
The cost of computing $\dpOmega(\btheta)$ is $\bigO(|\cE|)$, which is
exponentially better than $\bigO(|\cY|)$.

In many applications, we will often rather be interested in the
\textit{argument} that achieves the maximum, \textit{i.e.}, one of the
highest-scoring paths
\begin{equation}
\Y^\star(\btheta) \in \argmax_{\Y \in \cY} ~ \langle \Y, \btheta \rangle.
\label{eq:lp_argmax}
\end{equation}
This argument can be computed by \textit{backtracking}, that we now relate
to computing subgradients of $\lp(\btheta)$.

\paragraph{Linear program, lack of differentiality.}

Unfortunately, $\lp(\btheta)$ is not differentiable everywhere.  To see why
this is the case, notice that \eqref{eq:LP} can be rewritten as a linear program
over the convex polytope $\conv(\cY)$:
\begin{equation}
    \lp(\btheta) = \max_{\Y \in \conv(\cY)} ~ \langle \Y, \btheta \rangle.
\end{equation}
From the generalized Danskin theorem
\citep{bertsekas1971control},
\begin{equation}
\Y^\star(\btheta) \in 
\argmax_{\Y \in \conv(\cY)} ~ \langle \Y, \btheta \rangle
\subseteq \partial \lp(\btheta),
\end{equation}
where $\partial$ denotes the subdifferential of $\lp(\btheta)$, \textit{i.e.}, the set of
subgradients. When $\Y^\star(\btheta)$ is unique, $\partial \lp(\btheta)$ is a singleton
and $\Y^\star$ is equal to the gradient of $\lp(\btheta)$, that we write $\nabla
\lp(\btheta)$.  Unfortunately, $\Y^\star(\btheta)$ is not always unique, meaning that
$\lp(\btheta)$ is not differentiable everywhere. This hinders optimization
as we can only train models involving $\lp(\btheta)$ with 
\textit{subgradient} methods.
 Worse, $\Y^\star(\btheta)$,
a function from $\bTheta$ to $\cY$, is discontinuous and has null 
or undefined derivatives. It is thus impossible
to use it in a model trained by
gradient~descent.

\subsection{Smoothed max layers}

To address the lack of differentiability of dynamic programming, we introduce
the operator $\maxOmega$, presented in~\S\ref{sec:smoothed_max_op}, and consider two
approaches.

\paragraph{Smoothing the linear program.} 
Let us define the $\Omega$-smoothed maximum of a function $f \colon
\mathbb{\cY} \to \RR$ over a finite set $\cY$ using the following shorthand
notation:
\begin{equation}
\underset{\Y \in \cY}{\maxOmega} ~ f(\Y) \triangleq \maxOmega((f(\Y))_{\Y
\in \cY}).
\end{equation}
A natural way to circumvent the lack of differentiability of $\lp(\btheta)$ is
then to replace the \textit{global} $\max$ operator by $\maxOmega$:
\begin{equation}
    \lpOmega(\btheta) \triangleq \underset{\Y \in \cY}{\maxOmega} ~ \langle \Y,
    \btheta \rangle \in
\mathbb{R}.
\label{eq:LP_Omega}
\end{equation}
From \S\ref{sec:smoothed_max_op}, $\lpOmega(\btheta)$ is convex and, as long as
$\Omega$ is strongly convex, differentiable everywhere. In addition,
$\nabla \lpOmega(\btheta)$ is Lipschitz continuous and thus differentiable
almost everywhere.  Unfortunately, solving \eqref{eq:LP_Omega} for general
strongly convex $\Omega$ is intractable when $\cY$ has an exponential size.

\paragraph{Smoothing the dynamic program.} As a tractable alternative, we propose an 
\textit{algorithmic}
smoothing. Namely, we replace $\max$ by $\maxOmega$ \textit{locally} within the DP
recursion.
Omitting the dependence on $\Omega$, this defines a smoothed recursion over the
new sequence ${(v_i(\btheta))}_{i=1}^N$: 
\begin{gather}
v_1(\btheta) \triangleq 0 \\
\forall i \in [2,\dots, N]: ~
v_i(\btheta)\triangleq \underset{j \in \cP_i}{\maxOmega} ~ \theta_{i,j} +
v_j(\btheta).
\label{eq:smoothed_recursion}
\end{gather}
The new algorithm outputs $\dpOmega(\btheta)\,{\triangleq}\, v_N(\btheta)$, the
\textit{smoothed highest score}.  
Smoothing the max operator locally
brings the same benefit as before ---
$\dpOmega(\btheta)$ is smooth and $\nabla \dpOmega(\btheta)$ is differentiable
almost everywhere. However, computing $\dpOmega(\btheta)$ is now
always tractable, since it simply requires to evaluate ${(v_i(\btheta))}_{i=1}^N$
in topological order, as in the original recursion~\eqref{eq:general_recursion}.
Although $\lpOmega(\btheta)$ and $\dpOmega(\btheta)$ are generally different (in
fact, $\lpOmega(\btheta) \ge \dpOmega(\btheta)$ for all
$\btheta \in \bTheta$), we now show that $\dpOmega(\btheta)$ is a sensible
approximation of $\lp(\btheta)$ in several respects.
\begin{proposition}{Properties of $\dpOmega$}
    \label{proposition:dpOmega}

\begin{enumerate}[topsep=0pt,itemsep=3pt,parsep=3pt]
\item \label{property:dpOmega_convex} $\dpOmega(\btheta)$ is convex
\item \label{property:dpOmega_bound} $\lp(\btheta) - \dpOmega(\btheta)$ is
    bounded above and below:
\begin{equation}
(N-1) L_{\Omega,N} \le 
\lp(\btheta) - \dpOmega(\btheta) \le
(N-1) U_{\Omega,N},
\end{equation}
where $L_{\Omega,N}$ and $U_{\Omega,N}$ are defined in 
Lemma \ref{lemma:max_Omega_properties}.
\item When $\Omega$ is separable, $\dpOmega(\btheta)=\lpOmega(\btheta)$ 
\textbf{if and only if} $\Omega = - \gamma H$, where $\gamma \ge 0$. 
\end{enumerate}
\end{proposition}
Proofs are given in \S\ref{appendix:proof_dpOmega}.  The first claim can
be surprising due to the recursive definition of
$\dpOmega(\btheta)$.  The second claim implies that $\dpz_{\gamma
\Omega}(\btheta)$ converges to $\lp(\btheta)$ when the regularization vanishes:
$
    \dpz_{\gamma \Omega}(\btheta) \to_{\gamma \to 0} \lp(\btheta)
$;
$\lp_{\gamma \Omega}(\btheta)$ also satisfies this property.  The ``if''
direction of the third claim follows by showing that $\max_{-\gamma H}$
satisfies associativity. This recovers known results in the framework of message
passing algorithms for probabilistic graphical
models~\citep[\textit{e.g.,}][Section 4.1.3]{wainwright_graphical_2008}, with a
more algebraic point of view.  The key role that the distributive and
associative properties play into breaking down large problems into smaller ones
has long been noted \citep{verdu1987,aji_2000}.  However, the ``and only if''
part of the claim is new to our knowledge. Its proof shows that
$\max_{-\gamma H}$ is the only
$\maxOmega$ satisfying associativity, exhibiting a functional equation
from information theory~\citep{horibe_entropy_1988}.
While this provides an argument
in favor of entropic regularization, $\ell_2^2$ regularization has
different benefits in terms of sparsity of the solutions.

\subsection{Relaxed argmax layers}\label{sec:differentiation}

It is easy to check that $\nabla \lpOmega(\btheta)$ belongs to $\conv(\cY)$
and can be interpreted as an expected path under some distribution
induced by $\nabla \maxOmega$,
over all possible $\Y \in \cY$ 
--- see \S\ref{appendix:nabla_lp_omega} for details.
This makes $\nabla \lpOmega(\btheta)$ interpretable as a \textit{continuous
relaxation} of the highest-scoring path $\Y^\star(\btheta)$ defined in
\eqref{eq:lp_argmax}. However, like $\lpOmega(\btheta)$, computing $\nabla
\lpOmega(\btheta)$ is intractable in the general case.
Fortunately, we now show that $\nabla
\dpOmega(\btheta)$ is always easily computable by \textit{backpropagation} and
enjoys similar properties.

\paragraph{Computing $\nabla \dpOmega(\btheta)$.}

Computing $\nabla \dpOmega(\btheta)$ can be broken down into two steps.
First, we compute and record the local gradients alongside the recursive step
\eqref{eq:smoothed_recursion}:
\begin{equation}
    \forall\,i \in [N]:\quad\q_i(\btheta) \triangleq \nabla \maxOmega(\btheta_i
    + \v(\btheta)) \in \Simplex^N,
\end{equation}
where $\v(\btheta) \triangleq (v_1(\btheta),\dots,v_N(\btheta))$.
Since we assume that $\theta_{i,j}=-\infty$ if $(i,j) \not \in \cE$, we have
$\supp(\q_i(\btheta)) = \cP_i$.
This ensures that, similarly to~$v_i(\btheta)$, $\q_i(\btheta)$ exclusively
depends on $(v_j(\btheta))_{j \in \cP_i}$. 
Let $\cC_j$ be the children of node $j \in [N]$.
A straighforward application of backpropagation (cf.
\S\ref{appendix:proof_proposition_nabla_dpOmega}) yields a recursion run in
\textit{reverse-topological} order, starting from node $j=N-1$ down to $j=1$:
\begin{align}
\forall\,i \in \cC_j:\: 
e_{i,j} \leftarrow \bar e_i q_{i,j} \text{ then }
\bar e_j \leftarrow \sum_{i \in \cC_j} e_{i,j},\label{eq:gradient_recursion}
\end{align}
where $\bar e_N \leftarrow 1$ and $e_{i,j} \leftarrow 0$ for $(i, j) \notin
\cE$.  The final output is $\nabla \dpOmega(\btheta) = \E$.  Assuming
$\maxOmega$ can be computed in linear time, the total cost is $\bigO(|\cE|)$,
the same as $\dpz(\btheta)$.  Pseudo-code is summarized in 
\S\ref{appendix:proof_proposition_nabla_dpOmega}.

\paragraph{Associated path distribution.}
The backpropagation we derived has a probabilistic interpretation. Indeed, $\Q(\btheta) \in \RR^{N\times N}$ can be
interpreted as a transition matrix: it defines a \textit{random walk} on the
graph $G$, \textit{i.e.}, a finite Markov chain with states $\cV$ and
transition probabilities supported by $\cE$. The random walk starts
from node $N$ and, when at node~$i$, hops to node~$j \in \cP_i$ with probability
$q_{i, j}$. It always ends at node $1$, which
is absorbing. The walk follows the path $\Y \in \cY$ with a probability
$\proba{\Y}{\btheta,\Omega}$, which is simply
the product of the $q_{i,j}$ of visited edges. 
Thus, $\Q(\btheta)$ defines a \textit{path distribution} $p_{\btheta,\Omega}$.
Our next proposition shows that $\nabla \dpOmega(\Y) \in \conv(\cY)$ and
is equal to the expected path $\expect{\Y}{\btheta,\Omega}$ under that distribution.
\vspace{.3em}
\begin{proposition}{$\nabla \dpOmega(\btheta)$ as an expected path}
    \label{proposition:nabla_dpOmega}%
\begin{equation}
    \normalfont%
        \forall\,\btheta \in \bTheta:\quad\nabla \dpz_\Omega(\btheta) = 
        \expect{\Y}{\btheta,\Omega} = \E \in \conv(\cY).
\end{equation}
\end{proposition}
Proof is provided in \S\ref{appendix:proof_proposition_nabla_dpOmega}.
Moreover, $\nabla \dpOmega(\btheta)$ 
is a principled relaxation of the
highest-scoring path $\Y^\star(\btheta)$,
in the sense that it converges to a subgradient of
$\lp(\btheta)$ as the regularization vanishes:
\begin{equation}
\forall\,\btheta \in \bTheta:\quad
\nabla \dpz_{\gamma \Omega}(\btheta)
\xrightarrow[\gamma \to 0]{}
\Y^\star(\btheta) \in \partial \lp(\btheta).
\end{equation}
When $\Omega=-\gamma H$, the distributions underpinning $\lpOmega(\btheta)$ and
$\dpOmega(\btheta)$ coincide and reduce to the Gibbs distribution
$\proba{\Y}{\btheta,\Omega} \propto \exp(\langle \btheta, \Y \rangle / \gamma)$.
The value $\lpOmega(\btheta)=\dpOmega(\btheta)$ is then equal to the log
partition.  When $\Omega=\gamma\|\cdot\|^2$,
some transitions between nodes have zero probability and hence some paths have
zero probability under the distribution $p_{\btheta,\Omega}$.  Thus, $\nabla
\dpOmega(\btheta)$ is typically \textit{sparse} --- this will prove interesting
to introspect the various models we consider (typically, the smaller $\gamma$,
the sparser $\nabla \dpOmega(\btheta)$).

\subsection[Multiplication with the Hessian]{Multiplication with the Hessian $\nabla^2 \dpOmega(\btheta) \Z$}
\label{sec:hessian}

Using $\nabla \dpOmega(\btheta)$ as a layer involves
backpropagating through $\nabla\dpOmega(\btheta)$. This requires computing the
Jacobian $\nabla \nabla \dpOmega(\btheta)$ or in other words the Hessian
$\nabla^2 \dpOmega(\btheta)$, a linear map from $\RR^{N \times N}$ to $\RR^{N
\times N}$. Fortunately, a practical implementation of backpropagation only requires to apply that map to
a provided matrix~$\Z  \in \RR^{N \times N}$, \ie,
$\nabla^2\dpOmega(\btheta)\Z$. 
We therefore focus on that term.  Recall that the
directional derivative of $\dpOmega$ at $\btheta$ along $\Z$ can be computed by
$\langle \nabla \dpOmega(\btheta), \Z \rangle \in \RR$.
Our key technique, which is also at the heart of Pearlmutter's method
(\citeyear{pearlmutter_fast_1994}),
is to observe that $\nabla^2\dpOmega(\btheta)\Z$ is the
gradient of the directional derivative at $\btheta$ along $\Z$.  Namely,
\begin{equation}
\nabla^2 \dpOmega(\btheta) \Z = \nabla \langle \nabla \dpOmega(\btheta), \Z
\rangle.
\end{equation}
We therefore break down the computation of $\nabla^2 \dpOmega(\btheta) \Z$ into
two steps. First, we compute the directional derivative  $\langle \nabla
\dpOmega(\btheta), \Z \rangle$ using the chain rule. It can be computed in one
topologically-ordered pass over $G$. Similarly to the gradient computation, we
record multiplications with the (generalized) \textit{local} Hessian
$\H_i(\btheta) \triangleq \nabla^2 \maxOmega(\btheta_i + \v(\btheta))$ along the
way. Second, we compute the gradient of the directional derivative using
backpropagation. It yields a recursion for computing $\nabla^2 \dpOmega(\btheta)
\Z$ in reverse topological-order over~$G$. The complete derivation and the
pseudo-code are given in \S\ref{appendix:hessian_vector_product}.  The total
computational cost is $\bigO(|\cE|)$, as for the gradient computation.

 \paragraph{Performance.} 
 
 Using autodiff frameworks such as \textit{PyTorch}~\citep{paszke2017pytorch}, it is
 possible to only implement $\dpOmega(\btheta)$ and rely on
 tape-based gradient computation to obtain $\nabla \dpOmega(\btheta)$. Provided
 that we tape the backward pass as well, we can then backpropagate again through
 $\nabla \dpOmega(\btheta)$ to obtain $\nabla^2 \dpOmega(\btheta)\Z$. In
 practice, however, implementing backpropagation without resorting to
 autodiff software is crucial, since the DAG structure can be directly harcoded in
 concrete cases --- see \S\ref{sec:examples}.  Moreover, our
 ``reverse-over-forward'' approach to compute the Hessian product
 (backpropagating over the directional derivative computation) yields a simpler
 computation graph than the ``reverse-over-reverse'' approach (backpropagation
 over taped backpropagation).
In experiments, our approach 
 is up to $50 \times $ faster than vanilla \textit{PyTorch} on the Viterbi DAG.
 Note that our algorithms are readily vectorizable and can efficiently handle mini-batches with
  varying input lengths.

\paragraph{Summary.}

We have proposed $\dpOmega(\btheta)$, a smooth, convex and tractable
relaxation to the \textit{value} of $\lp(\btheta)$. We have also shown that $\nabla
\dpOmega(\btheta)$ belongs to $\conv(\cY)$ and is therefore a
sound relaxation to \textit{solutions} of $\lp(\btheta)$. To conclude this
section, we formally define our proposed two layers.
\begin{definition}{Differentiable dynamic programming layers}%
\begin{align}
    \text{Value layer:}&\quad\dpOmega(\btheta) \in \RR \\        
    \text{Gradient layer:}&\quad \nabla \dpOmega(\btheta) \in \conv(\cY)
\end{align}
\end{definition}

\section{Examples of computational graphs}
\label{sec:examples}

We now illustrate two instantiations of our framework for specific
computational graphs.

\subsection{Sequence prediction}\label{sec:seq_pred}

We demonstrate in this section how to instantiate $\dpOmega$ to the
computational graph of the
Viterbi algorithm~\citep{viterbi_error_1967,rabiner}, one of the most famous
instances of DP algorithm. We call the resulting operator $\viterbiOmega$.
We wish to tag a sequence $\X = (\x_1, \dots, \x_T)$ of vectors in $\RR^D$
(\textit{e.g.,} word representations) with the most probable output sequence (\textit{e.g.,}
entity tags) $\y = (y_1, \dots, y_T) \in [S]^T$.
This problem can be cast
as finding the highest-scoring path on a \textit{treillis} $G$.
While $\y$ can always be represented as a sparse $N \times N$ binary matrix,
it is convenient to represent it instead
as a $T \times S \times S$ binary tensor $\Y$, such that $y_{t,i,j} = 1$ if $\y$
transitions from node $j$ to node $i$ on time $t$, and~$0$ otherwise  --- we set $y_0 = 1$.
The potentials can similarly be organized as a $T \times S \times S$ real
tensor, such that $\theta_{t,i,j} = \phi_t(\x_t, i, j)$. 
Traditionally, the potential functions $\phi_t$ were
human-engineered \citep[\S 2.5]{sutton_introduction_2011}. In recent works
and in this paper,
they are learned end-to-end \citep{lample_2016}.

Using the above binary tensor representation, the inner product $\langle \Y,
\btheta \rangle$ is equal to $\sum_{t=1}^{T} \phi_t(\x_t, y_{t}, y_{t-1})$,
$\y$'s cumulated score.  This is illustrated in
Figure~\ref{fig:viterbi} on the task of part-of-speech tagging.
The bold arrows indicate one possible output sequence $\y$, \textit{i.e.},
one possible path in $G$.
 \begin{figure}[ht]
    \centering
    \includegraphics[height=0.2\textheight]{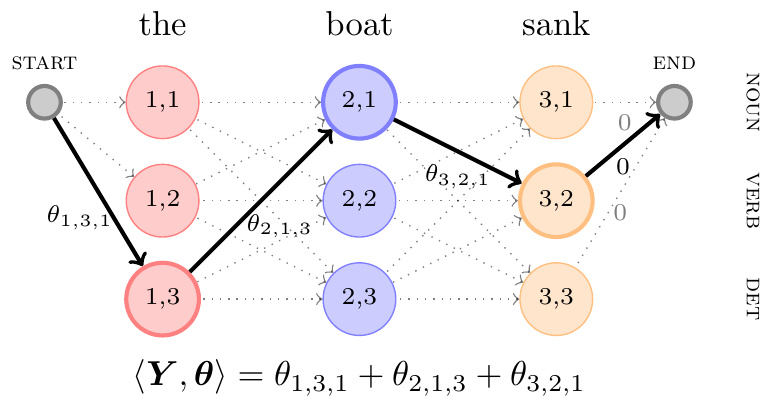}
    \caption{Computational graph of the Viterbi algorithm.}\label{fig:viterbi}
\end{figure}

When $\Omega = -H$, we recover linear-chain conditional random
fields (CRFs) \citep{lafferty_crf} and the probability of $\y$
($\Y$ in tensor representation) given $\X$ is
\begin{equation}
\proba{\y|\X}{\btheta,-H} 
{\propto} \exp(\langle \Y, \btheta \rangle)
{=}\exp\big(\sum_{t=1}^{T} \phi_t(\x_t, y_{t}, y_{t-1})\big).
\end{equation}
From Prop.~\ref{proposition:nabla_dpOmega}, the gradient  $\nabla
\text{Vit}_{-H}(\btheta) = \E \in~\RR^{T {\times} S {\times} S}$ is such that
$e_{t,i,j} = \proba{y_t=i,y_{t-1}=j|\X}{\btheta,-H}$.  The marginal probability
of state $i$ at time $t$
is simply $\proba{y_t=i|\X}{\btheta,-H} = \sum_{j=1}^S e_{t,i,j}$.  Using a
different $\Omega$ simply changes the distribution over state transitions. When
$\Omega=\|\cdot\|^2$, the marginal probabilities are typically \textit{sparse}.
Pseudo-code for $\viterbiOmega(\btheta)$,
as well as gradient and Hessian-product computations, 
is provided in~\S\ref{appendix:viterbi}.  The case $\Omega=\|\cdot\|^2$ is new
to our knowledge.

When $\Omega=-H$, the marginal probabilities are traditionally computed using
the forward-backward algorithm \citep{baum_1966}. In contrast, we compute
$\nabla \text{Vit}_{-H}(\btheta)$ using backpropagation while efficiently maintaining
the marginalization. An advantage of our
approach is that all operations are numerically stable.
The relation between forward-backward and backpropagation has been noted
before (\textit{e.g.,} \citet{eisner_inside-outside_2016}).
However, the analysis is led using $(+,\times)$ operations, instead
of $(\maxOmega,+)$ as we do.
This Viterbi instantiation can
immediately be generalized to graphical models with a tree structure, and to
approximate inference in general graphical models, since unrolled loopy belief
propagation~\citep{pearl_probabilistic_2014}
yields a dynamic~program.

\subsection{Time-series alignment}\label{sec:dtw}

We now demonstrate how to instantiate $\dpOmega$ to the computational graph
of dynamic time warping (DTW)
\citep{Sakoe78}, whose goal is to seek the \textit{minimal} cost alignment
between two time-series. We call the resulting operator $\dtwOmega$. Formally, let $N_A$ and
$N_B$ be the lengths of two time-series, $\A$ and $\B$. Let $\a_i$ and $\b_j$ be
the $i^{\text{th}}$ and $j^\text{th}$ observations of $\A$ and $\B$,
respectively.  Since edge weights only depend on child nodes, it is convenient
to rearrange $\Y$ and $\btheta$ as $N_A \times N_B$ matrices.  Namely, we
represent an alignment $\Y$ as a $N_A \times N_B$ binary matrix, such that
$y_{i,j}=1$ if $\a_i$ is aligned with $\b_j$, and $0$ otherwise. Likewise, we
represent $\btheta$ as a $N_A \times N_B$ matrix. A classical example is
$\theta_{i,j} = d(\a_i, \b_j)$, for some differentiable discrepancy measure~$d$.
We write $\cY$ the set of all monotonic alignment matrices, such that the path
that connects the upper-left $(1,1)$ matrix entry to the lower-right $(N_A,N_B)$
one uses only $\downarrow,\rightarrow,\searrow$ moves. 
The DAG associated with
$\cY$ is illustrated in Figure \ref{fig:dtw_graph} with $N_A=4$ and $N_B=3$ below.
\begin{figure}[ht]
\centering
\includegraphics[height=0.25\textheight]{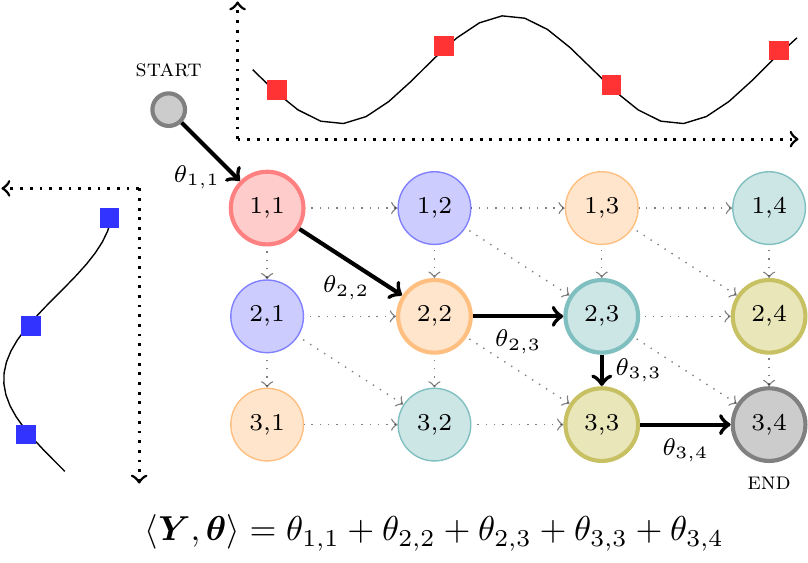}
\caption{Computational graph of the DTW algorithm.}
\label{fig:dtw_graph}
\end{figure}

Again, the bold arrows indicate one possible path $\Y \in \cY$ from start to end
in the DAG, and correspond to one possible alignment.  Using this
representation, the cost of an alignment (cumulated cost along the path) is
conveniently computed by $\langle \Y, \btheta \rangle$.
The value $\dtwOmega(\btheta)$ can be used to define a loss between
alignments or between time-series.
Following Proposition \ref{proposition:nabla_dpOmega}, $\nabla \dtwOmega(\btheta)
= \E \in \RR^{N_A \times N_B}$ can be understood as a soft alignment matrix.
This matrix is sparse when $\Omega=\|\cdot\|^2$, as
illustrated in Figure \ref{fig:smooth_dtw} (right).

Pseudo-code to compute $\dtwOmega(\btheta)$ as well as its
gradient and its Hessian products are provided
in~\S\ref{appendix:dtw}. When $\Omega=-H$, $\dtwOmega(\btheta)$ is a conditional random
field known as soft-DTW, and the
 probability $\proba{\Y|\A,\B}{\btheta,\Omega}$ is a Gibbs distribution similar to \S\ref{sec:seq_pred}
~\citep{soft_dtw}.  However, the case $\Omega=\|\cdot\|^2$ and the computation
of $\nabla^2 \dtwOmega(\btheta) \Z$ are new and allow new applications.

\section{Differentiable structured prediction}
\label{sec:structured_prediction}

We now apply the proposed layers, $\dpOmega(\btheta)$ and $\nabla
\dpOmega(\btheta)$, to structured prediction \citep{bakir_2007}, whose goal is
to predict a structured output $\Y \in \cY$ associated with a structured input
$\X \in \cX$.  We define old and new structured losses, and demonstrate them on
two structured prediction tasks: named entity recognition and
time-series alignment.

\subsection{Structured loss functions}

Throughout this section, we assume that the potentials $\btheta \in~\bTheta$ have already been computed using a function from $\cX$ to $\bTheta$ 
and let $C \colon \cY \times \cY \to \RR_+$ be a cost function between
the ground-truth output $\Ytrue$ and the predicted output $\Y$.

\paragraph{Convex losses.}

Because $C$ is typically non-convex, the cost-augmented structured hinge loss
\citep{structured_hinge} is often used instead for linear models
\begin{equation}
\ell_C(\Ytrue; \btheta) \triangleq
\max_{\Y \in \cY} ~ C(\Ytrue, \Y) + \langle \Y, \btheta \rangle - 
\langle \Ytrue, \btheta \rangle.
\label{eq:structured_hinge_loss}
\end{equation}
This is a convex upper-bound on $C(\Ytrue, \Y^\star(\btheta))$, where
$\Y^\star(\btheta)$ is defined in \eqref{eq:lp_argmax}.
To make the cost-augmented decoding tractable, it is usually assumed that
$C(\Ytrue, \Y)$ is linear in $\Y$, \ie, it can be written as
$\langle \C_\Ytrue, \Y \rangle$ for some matrix $\C_\Ytrue$. We can then rewrite
\eqref{eq:structured_hinge_loss} using our notation as
\begin{equation}
\ell_C(\Ytrue; \btheta) = \lp(\btheta + \C_\Ytrue) - \langle \Ytrue, \btheta \rangle.
\end{equation}
However, this loss function is non-differentiable. We therefore propose to
relax $\lp$ by substituting it with $\dpOmega$:
\begin{equation}
\ell_{C,\Omega}(\Ytrue; \btheta) \triangleq 
\dpOmega(\btheta + \C_\Ytrue) - \langle \Ytrue, \btheta \rangle.
\label{eq:ell_C_Omega}
\end{equation}
Losses in this class are convex, smooth, tractable for any $\Omega$, and by
Proposition \ref{proposition:dpOmega} property \ref{property:dpOmega_bound} a
sensible approximation of $\ell_C$.  In addition, they only require to
backpropagate through $\dpOmega(\btheta)$ at training time.
It is easy to check that we recover 
the structured hinge loss with
$\ell_{C,0}$ \citep{structured_hinge} and the CRF loss with
$\ell_{0,-H}$ \citep{lafferty_crf}. 
The last one has been used on top of LSTMs in several
recent works \citep{lample_2016,ma_2016}. Minimizing $\ell_{0,-H}(\btheta)$ is
equivalent to maximizing the likelihood $\proba{\Ytrue}{\btheta,-H}$.  However,
minimizing $\ell_{0,\|\cdot\|^2}$ is \textit{not} equivalent to maximizing
$\proba{\Ytrue}{\btheta,\|\cdot\|^2}$. In fact, the former is convex while the
latter is not.

\paragraph{Non-convex losses.}

A direct approach that uses the output distribution $\p_{\theta,\Omega}$
consists in minimizing the risk $\sum_{\y \in
\cY} \proba{\Y}{\btheta,-H} C(\Ytrue, \Y)$. As shown by
\citet{minimum_risk_training_stoyanov}, this can be achieved
by backpropagating through the minimum risk decoder. 
However, the risk
is usually non-differentiable, piecewise constant
\citep{minimum_risk_annealing} and several smoothing heuristics
are necessary to make the method work \citep{minimum_risk_training_stoyanov}.

Another principled approach is to consider a differentiable approximation  
$\Delta \colon \cY \times \conv(\cY) \to \RR_+$ of the cost $C$.
We can then relax $C(\Ytrue,
\Y^\star(\btheta))$ by $\Delta(\Ytrue, \nabla \dpOmega(\btheta))$.  
Unlike minimum risk training, this approach is differentiable everywhere when
$\Omega=-H$. Both approaches require to backpropagate through $\nabla
\dpOmega(\btheta)$, which is
only roughly twice as costly as backpropagating through $\dpOmega(\btheta)$ using
the approach outlined in~\S\ref{sec:hessian}.

\subsection{Named entity recognition}\label{sec:ner}

\label{sec:modeling_potentials}

Let $\X = (\x_1, \cdots, \x_T)$ be an input sentence, where each
word $\x_t$ is represented by a vector in $\RR^D$, computed using a
neural recurrent architecture trained end-to-end. We wish
to tag each word with named entities, \textit{i.e.}, identify blocks
of words that correspond to names, locations, dates, etc. We use the specialized operator
$\viterbiOmega$ described in \S\ref{sec:seq_pred}. In our experiments, 
we define the elements of the potential tensor
$\btheta(\X) \in \RR^{T \times S \times S}$ when $t > 1$ by
\begin{equation}\label{eq:linear_potential}
\theta(\X)_{t,i,j} \triangleq \w_i^\top \x_t + b_i + t_{i,j}
\end{equation}
and $\theta(\X)_{1, i, j} \triangleq \w_i^\top \x_t + b_i$, 
where $(\w_i,b_i) \in \RR^D \times \RR$ is the linear classifier associated with
tag $i$ and $\T \in \RR^{S \times S}$ is a transition matrix.
We learn $\W$, $\b$ and $\T$ along with the network producing $\X$,
and compare two losses:
\begin{equation}
\begin{aligned}
&\text{Surrogate convex loss:} &&\ell_{0,\Omega}(\Ytrue; \btheta), \\
&\text{Relaxed loss:} &&\Delta(\Ytrue, \nabla \dpOmega(\btheta)),
\end{aligned}
\end{equation}
where $\Delta(\Ytrue, \Y)$ is the squared $\ell_2$ distance when $\Omega = \Vert \cdot
\Vert_2^2$ and the Kullback-Leibler divergence when $\Omega = -H$, applied
row-wise to the marginalization of $\Ytrue$ and $\Y$.

\paragraph{Experiments.}

We measure the performance of the different losses and regularizations on the
four languages of the CoNLL 2003~\citep{conll} dataset. Following~\citet{lample_2016},
who use the $\ell_{0, -H}$ loss,
we use a character LSTM
and pretrained embeddings computed using \textit{FastText}~\citep{joulin2016fasttext} on
Wikipedia. Those are fed to a word bidirectional LSTM to
obtain $\X$. Architecture details are provided in~\S\ref{appendix:ner}.
Results are reported in Table \ref{table:ner_results}, along with \citep{lample_2016}
results with different pretrained embeddings. With proper parameter
selections, all losses perform within $1\%$ $F_1$-score of each other, although
entropy-regularized
losses perform slightly better on \nicefrac{3}{4} languages. However, the
$\ell_2^2$-regularized losses
yield sparse predictions, whereas entropy regularization always yields dense probability vectors.
Qualitatively, this allows to identify ambiguous predictions more easily ---
this is illustrated in \S\ref{appendix:ner} with additional figures.

\begin{table}[t]
    \caption{$F_1$ score comparison on CoNLL03 NER datasets.}\label{table:ner_results}
    \vspace{-.3em}
    \begin{center}
    \begin{small}
        \begin{tabular}{llrrrr}
            \toprule
            $\Omega$ &      Loss &  English &  Spanish &  German &  Dutch \\
            \midrule
            { Negentropy} &  { Surrogate} &            90.80 &    \textbf{86.68} &   77.35 &  \textbf{87.56} \\
                 &       { Relaxed}                    &        90.47 &    86.20 &   \textbf{77.56} &  87.37 \\
                 \midrule
            { $\ell_2^2$}    &  { Surrogate}   &    \textbf{90.86} &	85.51 & 76.01 & 86.58 \\
                             &       { Relaxed} &                   89.49 &    84.07 &   76.91 &  85.90 \\
                             \midrule
            \multicolumn{2}{c}{ \citep{lample_2016}} & \textit{90.96} & \textit{85.75} & \textit{78.76} & \textit{81.74} \\
            \bottomrule
            \end{tabular}
    \end{small}
    \end{center}        
    \end{table}

\subsection{Supervised audio-to-score transcription}\label{sec:audio}

We use our framework to perform supervised audio-to-score alignment on the
Bach 10 dataset \citep{duan2011soundprism}.
The dataset consists of 10 music pieces with audio tracks, MIDI transcriptions,
and annotated
alignments between them. We transform the audio tracks into a sequence of
audio frames using a feature extractor (see \S\ref{app:audio}) to
obtain a sequence $\A \in \RR^{N_A \times D}$, while the associated score
sequence is represented by $\B \in \RR^{N_B \times K}$ (each row $\b_j$ is a
one-hot vector corresponding to one key $b_j$). Each pair $(\A, \B)$ is associated to
an
alignment $\Ytrue \in \RR^{N_A \times N_B}$. As described in \S\ref{sec:dtw}, we
need to define a discrepancy matrix $\btheta \in \RR^{N_A \times N_B}$ between the
elements of the two sequences. 
We set the cost between an audio frame and a key
to be the log-likelihood of this key given a multinomial linear classifier. For all $i \in [N_A]$, we define
\begin{gather}\label{eq:classic_metric}
    \l_i \triangleq - \log(\softmax(\W^\top \a_i + \bm c)) \in \RR^K, \\
    \text{and }\forall\, j \in [N_B],\, \theta_{i, j}\triangleq l_{i, b_j},
\end{gather}
where $(\W,\c) \in \RR^{D \times K} \times \RR^K$ are learned classifier
parameters. We predict a soft alignment by $\Y = \nabla
\text{DTW}_{-H}(\btheta)$. Following
\citep{garreau_metric_2014}, we define the relaxed loss
\begin{equation}
    \Delta(\Ytrue, \Y) \triangleq \Vert \L (\Y - \Ytrue)^\top \Vert_F^2,
\end{equation}
where $\L$ a the lower triangular matrix filled with~$1$. When $\Y
\in \cY$ is a true
alignement matrix, $\Delta(\Ytrue, \Y)$ is the area between the path of $\Ytrue$
and $\Y$, which corresponds to the \textit{mean absolute deviation}
in the audio literature.
When $\Y \in \conv(\cY)$, it is a convex relaxation of the 
area. At test time, once $\btheta$ is learned, we use the non-regularized DTW
algorithm to output a hard alignment~$\Y^\star(\btheta) \in \cY$.

\paragraph{Results.}

We perform a
leave-one-out cross-validation of our model performance, learning the
multinomial classifier on $9$ pieces and assessing the quality of the alignment
on the remaining piece. We report the mean absolute deviation on both train and
test sets. A solid baseline consists in learning the multinomial classifier
$(\W,\c)$ beforehand, \textit{i.e.}, without end-to-end training.  We then use
this model to compute $\btheta$ as in~\eqref{eq:classic_metric} and obtain
$\Y^\star(\btheta)$. As shown in Table \ref{table:dtw}, our end-to-end technique
outperforms this baseline by a large margin. 
We also
demonstrate in
\S\ref{app:audio} that the alignments obtained by end-to-end training are
visibly closer to the ground truth.
End-to-end training thus allows to
\textit{fine-tune} the distance matrix $\btheta$ for the alignment task at
hand. 

\begin{table}[t]
    \vspace{-.8em}
    \caption{Mean absolute deviation of alignment using an end-to-end trained
        multinomial classifier and a pre-trained one.}\label{table:dtw}
    \vspace{.1em}
    \begin{center}
    \begin{small}
    \begin{tabular}{lcc}
        \toprule
        Linear model & Train & Test \\
        \midrule
        \textbf{End-to-end trained} & $\mathbf{0.17 \pm  0.01}$ & $\mathbf{1.07 \pm 0.61}$ \\
        Pretrained &  $1.80 \pm 0.14$ & $3.69 \pm  2.85$ \\
        Random $\btheta$ & $14.64 \pm 2.63$ & $14.64 \pm 0.29$ \\
        \bottomrule
        \end{tabular}
    \end{small}
    \end{center}
\end{table}

\section{Structured and sparse attention}
\label{sec:structured_attention}

We show in this section how to apply our framework to neural
sequence-to-sequence models augmented with an attention
mechanism~\citep{bahdanau_neural_2014}. An encoder first produces a list of
vectors $\X = (\x_1, \dots, \x_T)$ representing the input sequence. A decoder is
then used to greedily produce the corresponding output sequence. To simplify the
notation, we focus on one time step of the decoding procedure. Given the
decoder's current hidden state $\z$ and $\X$ as inputs, the role of the
attention mechanism is to produce a distribution $\w \in \Simplex^T$ over $\X$,
for the current time step. This distribution is then typically used to produce a
context vector $\c \triangleq \X^\top \w$, that is in turn invoved in the computation
of the output sequence's next element. 

\paragraph{Structured attention layers.}

\citet{kim_structured_2017} proposed a segmentation attention layer,
which is capable of taking into account the transitions between elements of
$\X$. They use a linear-chain CRF to model the probability
$\proba{\y|\X}{\btheta,-H}$ of a sequence $\y = (y_1,\dots,y_T)$,
where each $y_t$ is either $1$ (``pay attention'')
or $0$. They then propose to use normalized marginal probabilities as attention
weights: $w_t \propto \proba{y_t=1|\X}{\btheta,-H}$.  They show how to
backpropagate gradients through the forward-backward algorithm, which they use
to compute the marginal probabilities.

\paragraph{Generalizing structured attention.}

We now show how to generalize segmentation layers to any $\Omega$ and how to
backpropagate through them efficiently.  Using the notation from
\S\ref{sec:seq_pred}, any $\y$ can be represented as a tensor $\Y \in \{0,1\}^{T
\times 2 \times 2}$ and the potentials as a tensor $\btheta \in \RR^{T \times 2
\times 2}$. As in \citep{kim_structured_2017}, we define
\begin{equation}
\theta_{t,1,j} \triangleq \x_t \M \z + t_{1,j}
\quad \text{and} \quad
\theta_{t,0,j} \triangleq t_{0,j},
\end{equation}
where $\x \M \z$ is a learned bilinear form and $\T \in \RR^{2 \times 2}$ is
a learned transition matrix. Following \S\ref{sec:seq_pred}, the gradient
$\nabla \viterbiOmega(\btheta)$ is equal to the expected matrix $\E \in \RR^{T
\times 2 \times 2}$ and the marginals are obtained by marginalizing that matrix.
Hence, we can set $w_t \propto \proba{y_t=1|\X}{\btheta,\Omega} = e_{t,1,0} +
e_{t,1,1}$. 

Backpropagating through $\nabla \viterbiOmega(\btheta)$ can be carried out using
our approach outlined in \S\ref{sec:hessian}. This approach is not only more
general, but also simpler and more robust to underflow problems than
backpropagating through the forward-backward algorithm  as done in
\citep{kim_structured_2017}. 

\paragraph{Experiments.} We demonstrate structured attention layers with an LSTM
encoder and decoder to perform French to English translation using data from a 1
million sentence subset of the WMT14 FR-EN challenge. We illustrate an example of
attenion map obtained with negentropy and $\ell_2^2$ regularizations in
Figure~\ref{fig:attention}. Non-zero elements are underlined with borders:
$\ell_2^2$-regularized attention maps are sparse and more interpretable --- this
provides a structured alternative to sparsemax attention ~\citep{sparsemax}.
Results were all within~$0.8$ point of BLEU score on the newstest2014 dataset.
For French to English, standard softmax attention
obtained
\textbf{27.96}, while entropy and $\ell_2^2$ regularized attention obtained
\textbf{27.96} and \textbf{27.19} --- introducing structure and sparsity
therefore provides enhanced interpretability with comparable peformance.
We provide model details, full results and further visualizations in \S\ref{appendix:attention}.

\begin{figure}[t]
    \centering
    \includegraphics[width=.7\linewidth]{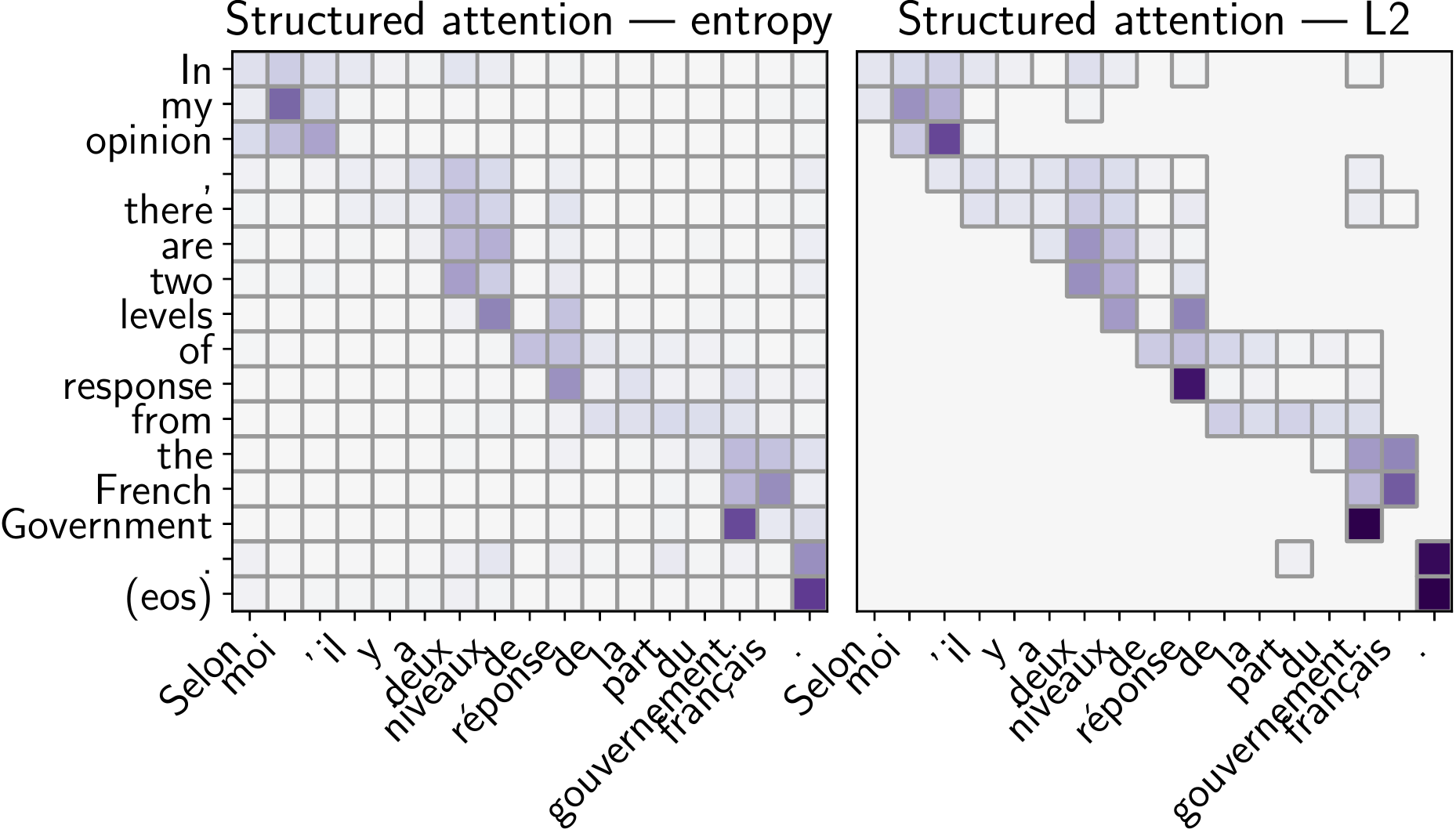}
    \caption{Attention maps obtained with structured attention. Although
        both regularizations led to the same translation ($y$-axis) in this
        example,
        attention is sparse and more interpretable with $\ell_2^2$.}
        \label{fig:attention}
\end{figure}

\section{Conclusion}

We proposed a theoretical framework for turning a broad class
of dynamic programs into convex, differentiable and tractable operators, using
the novel point of view of smoothed max operators.  Our work sheds a new light
on how to transform dynamic programs that predict hard assignments
(\textit{e.g.,} the maximum a-posteriori estimator in a probabilistic graphical
model or an alignment matrix between two time-series) into continuous and probabilistic ones.  We provided a new argument in
favor of negentropy regularization by showing that it is the only one to
preserve \textit{associativity} of the smoothed max operator.  We showed that
different regularizations induce different distributions over outputs and that $\ell_2^2$
regularization has other benefits, in terms of sparsity of the expected outputs.  Generally
speaking, performing inference in a graphical model and backpropagating through
it reduces to computing the first and second-order derivatives of a relaxed
maximum-likelihood estimation --- leveraging this observation yields elegant and efficient
algorithms that are readily usable in deep learning frameworks, with various
promising applications.

\section*{Acknowledgements}

MB thanks Vlad Niculae and Marco Cuturi for many fruitful discussions. AM thanks Julien
Mairal, Inria Thoth and Inria Parietal for lending him the computational
resources necessary to run the experiments. He thanks University Paris-Saclay and his Ph.D. supervisors
Bertrand Thirion and Ga\"el Varoquaux for allowing him to do an internship at NTT, and Olivier Grisel for his
insightful comments.


\clearpage

\appendix

\begin{center}
    {\Huge \bf Appendix}
\end{center}

\section{Proofs and detailed derivations}

This section contains the proofs of the propositions and lemmas
presented in the main text. It also contains derivations
of gradient, directional derivative and Hessian-product computations.

\subsection{Proof of Lemma \ref{lemma:max_Omega_properties} (properties of
$\maxOmega$)}
\label{appendix:proof_lemma_max_Omega_properties}

\paragraph{Property \ref{property:maxOmega_bound} (boundedness).} 
Let $\q^\star$ and
$\q^\star_\Omega$ be the solutions of $\max_{\q \in \Simplex^D} \q^\top \x$ and
$\max_{\q \in \Simplex^D} \q^\top \x - \Omega(\q)$, respectively.  Then, we have
\begin{equation} \maxOmega(\x) 
= \langle \q^\star_\Omega, \x \rangle - \Omega(\q^\star_\Omega) 
\ge \langle \q^\star, \x \rangle - \Omega(\q^\star) 
= \max(\x) - \Omega(\q^\star)
\end{equation}
and
\begin{equation}
\max(\x) - \Omega(\q^\star_\Omega) 
\ge \langle \q^\star_\Omega, \x \rangle - \Omega(\q^\star_\Omega)
= \maxOmega(\x).
\end{equation}
Combining the two and using $L_{\Omega,D} \le \Omega(\q) \le U_{\Omega,D} ~\forall \q
\in \Simplex^D$, we obtain
\begin{equation}
    \max(\x) - U_{\Omega,D} \le
\max(\x) - \Omega(\q^\star) \le 
\maxOmega(\x) \le \max(\x) - \Omega(\q^\star_\Omega) \le
\max(\x) - L_{\Omega,D}.
\end{equation}
When $\Omega(\q) = \sum_i q_i \log q_i$, we have the tight inequality $-\log D
\le \Omega(\q) \le 0 ~ \forall \q \in \Simplex^D$ and hence
\begin{equation}
    \max(\x) \le \maxOmega(\x) \le \max(\x) + \log D.
\end{equation}
When $\Omega(\q) = \frac{1}{2} \|\q\|^2$, we have the tight inequality
$\frac{1}{2D} \le \Omega(\q) \le \frac{1}{2} ~ \forall \q \in \Simplex^D$ and hence
\begin{equation}
\max(\x) - \frac{1}{2} \le \maxOmega(\x) \le \max(\x) - \frac{1}{2D}.
\end{equation}

Note that the difference $U_{\Omega,D} - L_{\Omega,D}$ is equal to $\log D$ when
$\Omega$ is the negative entropy and to $\frac{D-1}{2D} \le \frac{1}{2}$ when
$\Omega$ is the squared $\ell_2$ norm. Since $\log D > \frac{1}{2}$ for all
integers $D \ge 2$, we get a better approximation of the $\max$ operator using squared
$\ell_2$ norm than using negative entropy, whenever $D \ge 2$.

\paragraph{Property \ref{property:maxOmega_distrib} (distributivity of $+$ over
$\maxOmega$).} This follows immediately from 
\begin{equation}
\maxOmega(\x + c\ones) 
= \displaystyle{\max_{\q \in \Simplex^D}} \langle \q, \x + c\ones \rangle - \Omega(\q)
= \displaystyle{\max_{\q \in \Simplex^D}} \langle \q, \x \rangle - \Omega(\q) + c
= \maxOmega(\x) + c.
\end{equation}
Using our shorthand notation, this simply becomes $\underset{\Y \in \cY}{\maxOmega}
~ (f(\Y) + c) = \left(\underset{\Y \in \cY}{\maxOmega} ~ f(\Y)\right) + c$.

\paragraph{Property \ref{property:maxOmega_comm} (commutativity).}
Assume $\Omega(\P \q) = \Omega(\q)$ for all permutation matrices $\P$.
Let $\P^{-1}$ be the inverse permutation matrix associated with $\P$. 
Then we have
\begin{align}
\maxOmega(\P \x) 
&= \max_{\q \in \Simplex^D} \langle \q, \P \x \rangle - \Omega(\q)
= \max_{\q \in \Simplex^D} \langle \P^{-1} \q, \x \rangle - \Omega(\q) \\
&= \max_{\q \in \Simplex^D} \langle \q, \x \rangle - \Omega(\P \q)
= \max_{\q \in \Simplex^D} \langle \q, \x \rangle - \Omega(\q).
\end{align}

\paragraph{Property \ref{property:maxOmega_non_decreasing} (non-decreasingness
in each coordinate).}
If $\x \le \y$, then for all $\q \in \Simplex^D$, $\langle \x, \q\rangle  - \Omega(\q)\le
\langle \y, \q \rangle  - \Omega(\q)$,
as all $\q$ coordinates are non-negative. Thus $\maxOmega(\x) \le \maxOmega(\y)$.

\paragraph{Property \ref{property:nabla_maxOmega_j_equal_0} (insensitivity to
$-\infty$).}
Since $\maxOmega(\x) = \max_{\q \in \Simplex^D} \langle \q, \x \rangle -
\Omega(\q)$, if $x_j {=}-\infty$, then $q_j = \nabla \maxOmega(\x)_j = 0$ is the
only feasible solution for the $j$\textsuperscript{th} coordinate.

\subsection{Proof of Proposition \ref{proposition:recursion} (optimality of DP
recursion)}
\label{appendix:proof_proposition_recursion}

Let $v_i(\btheta)$ be the highest-score path up to node $i \in [N]$.  Let
$\cY_i$ be the set of paths $\y = (y_1, \dots, y_L)$ starting from node $1$ and
reaching node $i$, that is $y_1 = 1$ and $y_L = i$. Note that $L$ may depend on
$\y$ but we do not make this dependency explicit. Because nodes are sorted in
topological order, we can compute $v_i(\btheta)$ by 
\begin{equation}
v_i(\btheta) = \max_{\y \in \cY_i} ~ \sum_{t=2}^{L} \theta_{y_t,y_{t-1}}
= \max_{\y \in \cY_i} ~ 
\sum_{t=2}^{L-1} \theta_{y_t,y_{t-1}} + \theta_{y_L, y_{L-1}}
= \max_{\y \in \cY_i} ~ 
\sum_{t=2}^{L-1} \theta_{y_t,y_{t-1}} + \theta_{i, y_{L-1}}.
\end{equation}
Recall that $\cP_i$ is the set of parent nodes of node $i$. From the \textit{associativity}
of the max operator,
\begin{equation}
v_i(\btheta) 
= \max_{j \in \cP_i} \max_{\substack{\y \in \cY_i\\ y_{L-1}=j}}
~ \left( \sum_{t=2}^{L-1} \theta_{y_t,y_{t-1}} + \theta_{i, y_{L-1}} \right)
= \max_{j \in \cP_i} \max_{\substack{\y \in \cY_i\\ y_{L-1}=j}}
~ \left( \sum_{t=2}^{L-1} \theta_{y_t,y_{t-1}} + \theta_{i, j} \right).
\end{equation}
From the \textit{distributivity} of $+$ over $\max$, we obtain
\begin{equation}
v_i(\btheta) 
= \max_{j \in \cP_i} \left( \max_{\substack{\y \in \cY_i\\ y_{L-1}=j}}
~ \sum_{t=2}^{L-1} \theta_{y_t,y_{t-1}} \right) + \theta_{i, j}
= \max_{j \in \cP_i} ~ v_j(\btheta) + \theta_{i,j},
\end{equation}
where we used the fact that the inner max operations are independent of $y_L = i$.
This concludes the proof of the optimality of \eqref{eq:general_recursion}.

\subsection[Proof of Proposition dpOmega]{Proof of Proposition \ref{proposition:dpOmega} (properties of
    $\dpOmega(\btheta)$)}
\label{appendix:proof_dpOmega}

We prove in this section the three main claims of Proposition
\ref{proposition:dpOmega}. For the first two claims, we 
rewrite \eqref{eq:general_recursion} and \eqref{eq:smoothed_recursion} 
using the following notations:
\begin{align}
v_i^0(\btheta) &\triangleq \max(\u_i^0(\btheta)) 
\quad \text{and} \quad
v_i^\Omega(\btheta) \triangleq \max(\u_i^\Omega(\btheta)),
\quad \text{where} \\
\u_i^0(\btheta) &\triangleq (\theta_{i,1} +  v_1^0(\btheta), \dots,
\theta_{i,i-1} + v_{i-1}^0(\btheta), 
-\infty, -\infty, \dots, -\infty) \in \RR^N \quad \text{and} \\
\u_i^\Omega(\btheta) &\triangleq (\theta_{i,1} + v_1^\Omega(\btheta), \dots,
\theta_{i,i-1} + v_{i-1}^\Omega(\btheta), 
\underbrace{-\infty}_{i}, -\infty, \dots, -\infty) \in \RR^N.
\end{align}
These definitions are indeed valid as per Lemma \ref{lemma:max_Omega_properties},
property \ref{property:nabla_maxOmega_j_equal_0}.

\paragraph{Proof of $\dpOmega(\btheta)$ convexity.}

Since $v_1^\Omega(\btheta) = 0$, it is trivially convex.
Assume that $v_2^\Omega(\btheta), \dots, v_{i-1}^\Omega(\btheta)$ are convex.
Then, $v_i^\Omega(\btheta)$ is the composition of $\maxOmega$ and $\u_i^\Omega$,
a convex function and a function which outputs a vector whose each coordinate is
convex in $\btheta$. By induction, since
$\maxOmega$ is non-decreasing per coordinate (cf. Lemma
\ref{lemma:max_Omega_properties} property
\ref{property:maxOmega_non_decreasing}), $v_i^\Omega(\btheta)$ is convex \citep[\textit{e.g.,}][\S
3.2.4]{boyd2004convex}.  
Therefore $v_i^\Omega(\btheta)$ is convex for all $i \in [N]$ and 
$\dpOmega(\btheta) = v_N^\Omega(\btheta)$ is convex.

\paragraph{Proof of $\dpOmega(\btheta)$ bound.}

We clearly have $v_1^\Omega(\btheta) \ge v_1^0(\btheta)$.  
Assume that $v_j^\Omega(\btheta) \ge v_j^0(\btheta) - (j-1) U_{\Omega,N}$
for all $j \in \{2,\dots,i-1\}$. 
That is, $\u_i^\Omega(\btheta) \ge
\u_i^0(\btheta) - (i-2) U_{\Omega,N} \ones$, where $\ones \in \RR^N$ is
the unit vector. Then, by induction, we have
\begin{equation}
\maxOmega(\u_i^\Omega(\btheta)) \ge
\maxOmega(\u_i^0(\btheta)) - (i-2) U_{\Omega,N} \ge
\max(\u_i^0(\btheta)) - (i-1) U_{\Omega,N},
\end{equation}
where we used Lemma \ref{lemma:max_Omega_properties}, properties
\ref{property:maxOmega_bound}, \ref{property:maxOmega_distrib} and
\ref{property:maxOmega_non_decreasing}.
Therefore $v_i^\Omega(\btheta) \ge v_i^0(\btheta) - (i-1)U_{\Omega,N}$ for all $i
\in [N]$ and hence,
$\dpOmega(\btheta) \ge \lp(\btheta) - (N-1) U_{\Omega,N}$.
Using a similar reasoning we obtain
$v_i^0(\btheta) - (i-1) L_{\Omega,N} \ge v_i^\Omega(\btheta)$
and therefore $\lp(\btheta) - (N-1) L_{\Omega,N} \ge \dpOmega(\btheta)$.
To summarize, we obtain
\begin{equation}
    \lp(\btheta) - (N-1) L_{\Omega,N} \ge 
\dpOmega(\btheta) \ge
\lp(\btheta) - (N-1) U_{\Omega,N},
\end{equation}
which concludes the proof. Note that using property
\ref{property:maxOmega_bound} of~Lemma \ref{lemma:max_Omega_properties}, this immediately implies a bound involving
$\lpOmega(\btheta)$ instead of $\lp(\btheta)$.

\paragraph{Proof that $\Omega=-\gamma H \Rightarrow \dpOmega(\btheta) =
\lpOmega(\btheta)$. }

We first show that $\maxOmega$ is associative.
\begin{lemma}{Associativity of $\maxOmega$ when $\Omega=-\gamma H$}

We have $\maxOmega(\maxOmega(\x), c) = \maxOmega(\x, c)
\quad \forall \x \in \RR^D, c \in \RR$.
\label{lemma:associative_maxH}
\end{lemma}
\begin{proof}We simply use the closed form of $\maxOmega$ when $\Omega = - \gamma H$
    (cf. \S\ref{appendix:omega_examples}):
\begin{align}
\maxOmega(\maxOmega(\x), c) 
&= \gamma \log(\exp(\maxOmega(\x) / \gamma) + \exp(c / \gamma)) \\
&= \gamma \log\left(\exp\left(\log \sum_{i=1}^D \exp(x_i / \gamma)\right) +
\exp(c / \gamma)\right) \\
&= \gamma \log \left(\sum_{i=1}^D \exp(x_i / \gamma) + \exp(c / \gamma)\right) \\
&= \maxOmega(\x, c),
\end{align}
and the lemma follows.
\end{proof}
Using our shorthand notation, Lemma \ref{lemma:associative_maxH} can be used to
write
\begin{equation}
\underset{(y_1, \dots, y_i, \dots, y_L)}{\maxOmega} f(\y) = 
\underset{v}{\maxOmega} \underset{(y_1, \dots, v, \dots, y_L)}{\maxOmega} f(\y).
\end{equation}
This is precisely the associative property that we used in the proof of
Proposition \ref{proposition:recursion}.  The second property that we used, the
distributivity of $+$ over $\max$, holds for any $\maxOmega$, as per Lemma
\ref{lemma:max_Omega_properties} property \ref{property:maxOmega_distrib}. Thus,
the same proof as Proposition \ref{proposition:recursion} is also valid when we
substitute $\max$ with $\maxOmega$, when $\Omega=-\gamma H$, which yields
$\lpOmega(\btheta) = \dpOmega(\btheta)$.

\paragraph{Proof that $\Omega=-\gamma H \Leftarrow \dpOmega(\btheta) =
\lpOmega(\btheta)$. }Mirroring the previous proof, we first characterize the regularizations $\Omega$
for which $\maxOmega$ is associative.

\begin{lemma}
    
Let $\Omega \colon \Simplex^D \to \RR$ be a regularization function,
\ie, $\dom \Omega = \Simplex^D$.
Assume that there exist $\omega$ convex lower-semi-continuous defined on $[0,1]$
such that $\Omega(\q) = \sum_{i=1}^d \omega(q_i)$. If
\begin{equation}
    \maxOmega(\maxOmega(\x), c) = \maxOmega(\x, c)
\quad \forall \x \in \RR^D, c \in \RR,
\end{equation}
then 
$\Omega(\q) = - \gamma \sum_{i=1}^d q_i \log(q_i)$
for some $\gamma \geq 0$.
\end{lemma}

\begin{proof}We start by writing the associativity property
     for three elements. For all $x_1, x_2, x_3 \in \RR$,
\begin{align}
    \maxOmega&\big((x_1, x_2, x_3)\big) =
    \maxOmega\big(\maxOmega(x_1, x_2), x_3)\big) \\
    &= \max_{\substack{q + q_3 = 1 \\
    q, q_3 \ge 0}}
     q\,\max_{\substack{\tilde q_1 + \tilde q_2 = 1 \\
     \tilde q_i \ge 0}} \big(
     \tilde q_1 x_1 + \tilde q_2 x_2 - \omega(\tilde q_1)
     - \omega(\tilde q_2) \big)
    + q_3 x_3
    - \omega(q_3) - \omega(q) \\ 
    &= \max_{\substack{q_1 + q_2 + q_3 = 1 \\
    q_i \ge 0}}
    q_1 x_1 + q_2 x_2 + q_3 x_3 - \Phi(q_1, q_2, q_3),\quad\text{where} \\
    \Phi(q_1, q_2, q_3) &\triangleq (q_1 + q_2)
    \Big(\omega\big(\frac{q_1}{q_1 + q_2}\big) + 
    \omega\big(\frac{q_2}{q_1 + q_2}\big) \Big)
    + \omega(q_1 + q_2) + \omega(q_3).
\end{align}
We have performed a variable change $q_{1, 2} =  q\, \tilde q_{1, 2}$
at the second line, and noticed $q = q_1 + q_2$. Therefore
\begin{equation}
    \maxOmega\big((x_1, x_2, x_3)\big) = \Phi^\star(x_1, x_2, x_3),
\end{equation}
where $\Phi^\star$ is the convex conjugate of $\Phi$ restricted
to $]0, 1]^3$. By definition, we also have
$\maxOmega\big((x_1, x_2, x_3)\big) = \Omega^\star(x_1, x_2, x_3)$,
so that $\Omega^\star = \Phi^\star$ on $\RR^3$. As $\Omega$ is convex
and lower semi-continous, we can apply Moreau-Yoshida theorem and obtain
$\Omega^{\star\star} = \Omega = \Phi^{\star\star} \le \Phi$.

Suppose that there exists $\q = (q_1, q_2, q_3) \in \Simplex^3$ such that $\Phi(q_1, q_2, q_3) < \Omega(q_1, q_2, q_3)$.
Given the forms of $\Phi$ and $\Omega$, $\Phi(q_1, q_2, 0) < \Omega(q_1, q_2, 0)$. 
We let $\x = (x_1, x_2, -\infty) \in \RR^3$ such that
\begin{align}
    \maxOmega(x_1, x_2, -\infty) &= \maxOmega(x_1, x_2) = x_1 q_1 + x_2 q_2 - \omega(q_1) - \omega(q_2)
    = \langle \x, \q \rangle - \Omega(\q) \\
    &< \langle \x, \q \rangle - \Phi(\q)
    \leq \max_{\q \in \Simplex^3} \langle \x, \q \rangle - \Phi(\q) = \maxOmega\big(\maxOmega(x_1, x_2), -\infty)\big),
\end{align}
leading to a contradiction. Therefore $\Omega \geq \Phi$ over $\Simplex^3$, and finally $\Omega = \Phi$.
We have used the fact that the operator $\nabla\maxOmega : \RR^2 \to \Simplex^2$
 is surjective, as $\Simplex^2$ is a one-dimensional segment,
 $\nabla\maxOmega$ is continuous and reaches the extreme values $\nabla\maxOmega(0, -\infty) = (1, 0)$ and
$\nabla\maxOmega(-\infty, 0) = (0, 1)$ --- which allows to use the intermediate value theorem.

To conclude, for all $q_1, q_2 \in ]0, 1]$ such that $q_1 + q_2 \leq 1$, we have
\begin{align}
    \omega(q_1) + \omega(q_2) &= (q_1 + q_2)
    \Big(\omega\big(\frac{q_1}{q_1 + q_2}\big) + 
    \omega\big(\frac{q_2}{q_1 + q_2}\big) \Big)
    + \omega(q_1 + q_2) \\ 
    \omega(xy) + \omega((1 - x)y) - \omega(y) &= y(\omega(x) + \omega(1-x))
    \quad \forall\,0 < y \leq 1,\, 0 < x < 1,\label{eq:functional}
\end{align}
where we have set $y = q_1 + q_2$ and $x = \frac{q_1}{q_1 + q_2}$. The
functional equation~\eqref{eq:functional} was first studied in the field
of information theory. As first shown by~\citet[Theorem 0]{horibe_entropy_1988},
and further extended~\citep{gselmann_entropy_2011},
all measurable solutions have the form
\begin{equation}
    \omega(x) = - \gamma x \log(x),
\end{equation}
where $\gamma \ge 0$ is a constant. The lemma follows.
\end{proof}

Assuming that $\Omega$ is not equal to $- \gamma H$ for any $\gamma \ge 0$, the previous lemma
tells us that the associativity property is not met for a triplet $(x_1, x_2, x_3) \in \RR^3$.
In Figure~\ref{fig:example}, we construct a graph~$G$ such that
\begin{equation}
    \dpOmega(\btheta) = \maxOmega(\maxOmega(x_1, x_2), x_3) \neq \lpOmega(\btheta) = \maxOmega(x_1, x_2, x_3)   
\end{equation}
The proposition follows.
\begin{figure}[ht]
    \centering
    \includegraphics[width=.4\textwidth]{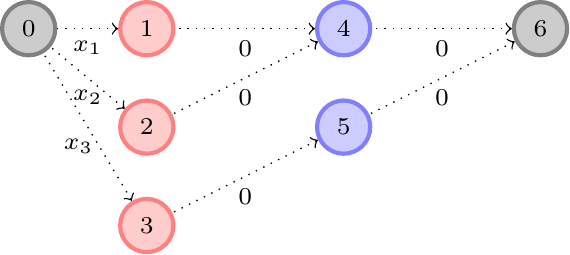}
    \caption{In general, $v_6(\btheta) = \dpOmega(\btheta) \neq \lpOmega(\btheta)$.}\label{fig:example}
\end{figure}

\subsection[Computation of gradient]{Computation of $\nabla \lpOmega(\btheta)$ and 
interpretation as an expectation}
\label{appendix:nabla_lp_omega}

We show that $\nabla \lpOmega(\btheta) \in \conv(\cY)$, and
characterize a path distribution of which $\nabla \lpOmega(\btheta)$ 
is the expectation.

\paragraph{Convex hull of $\cY$.}

We rewrite
$\lpOmega(\btheta) = \maxOmega(\u(\btheta))$, where $\u(\btheta) \triangleq
(\langle \Y, \btheta \rangle)_{\Y \in \cY}$. Using the chain rule, we have
\begin{equation}
\nabla \lpOmega(\btheta) = \J_{\u}(\btheta)^\top \nabla
\maxOmega(\u(\btheta)),
\label{eq:lpOmega_gradient}
\end{equation}
where $\J_{\u}$ is the Jacobian of $\u$ \wrt $\btheta$, a matrix of size $|\cY|
\times (N \times N)$. The horizontal slices of $\J_{\u}$ are exactly all the paths
$\Y$ of $\cY$. Using $\nabla \maxOmega(\u(\btheta)) \in
\Simplex^{|\cY|}$, we conclude that $\nabla \lpOmega(\btheta) \in \conv(\cY)$.

\paragraph{Induced distribution.}

From \eqref{eq:lpOmega_gradient}, we see that $\nabla \lpOmega(\btheta) =
\sum_{\Y \in \cY} \proba{\Y}{\btheta,\Omega} ~ \Y$, where we defined the distribution
\begin{equation}
\proba{\Y}{\btheta,\Omega} \triangleq \Big(\nabla \maxOmega(\u(\btheta)) \Big)_\Y.
\end{equation}
Unfortunately, since $\u(\btheta) \in \RR^{|\cY|}$, computing
$\proba{\Y}{\btheta,\Omega}$, let alone the expectation
$\expect{\Y}{\btheta,\Omega}$
under that distribution, is intractable for general $\Omega$.

\subsection[Proof nabla dpOmega]{Proof of Proposition \ref{proposition:nabla_dpOmega} (computation of
$\nabla \dpOmega(\btheta)$)
\label{appendix:proof_proposition_nabla_dpOmega}}

\paragraph{Gradient computation.}We first derive the recursion over $\E \triangleq \nabla \dpOmega(\btheta)$ using sensitivity analysis,
a.k.a backpropagation calculus.  For any
$(i,j) \in \cE$, since $\theta_{i,j}$ influences only $v_i$, a straighforward
application of the chain rule gives
\begin{equation}
e_{i,j} =
\partialfrac{v_N}{\theta_{i, j}} = 
\myblue{\partialfrac{v_N}{v_i}} \mygreen{\partialfrac{v_i}{\theta_{i, j}}}.
\label{eq:e_ij}
\end{equation}
Recall that $\v = (v_1, \dots, v_N)$ and
$\q_i \triangleq \nabla \maxOmega(\btheta_i + \v)$.
With this vector defined, we can now easily derive the two terms on the r.h.s of
\eqref{eq:e_ij}.  Differentiating
\eqref{eq:smoothed_recursion} \wrt $\theta_{i,j}$ straighforwardly gives the second term
$\mygreen{\partialfrac{v_i}{\theta_{i,j}}} = q_{i,j}$.

The first term must be computed recursively. Recall that $\cC_j$ denotes the children
of node $j$. Since a node $j$ influences only its children $i \in \cC_j$, using
the chain rule, we get
\begin{equation}
\partialfrac{v_N}{v_j} = 
\sum_{i \in \cC_j} \myblue{\partialfrac{v_N}{v_i}}
\partialfrac{v_i}{v_j} \triangleq \bar e_j.
\label{eq:v_N_v_j}
\end{equation}
Differentiating \eqref{eq:smoothed_recursion} \wrt $v_j$ again gives
$\partialfrac{v_i}{v_j} = q_{i,j}$. By definition, we also have
$\myblue{\partialfrac{v_N}{v_i}} = \bar e_i$ and $e_{i,j}= \bar e_i q_{i,j}$.
Hence,
\begin{equation}
\bar e_j = \sum_{i \in \cC_j} \bar e_i q_{i,j} = \sum_{i \in \cC_j} e_{i,j}.
\end{equation}
Combining the above, for any $j \in [N-1]$, we obtain the following two-step
recursion
\begin{equation}
\forall\,i \in \cC_j,\,e_{i,j} = \bar e_i q_{i,j} ~
\quad \text{and} \quad
\bar e_j = \sum_{i \in \cC_j} e_{i,j}.
\label{eq:backward}
\end{equation}
The values ${(e_{i, j})}_{(i, j) \in \cE}$ can thus be computed in
reverse topological order over the nodes of $G$, initializing $\bar e_N =
\partialfrac{v_N}{v_N} = 1$.  
The pseudo-code is summarized in Algorithm \ref{alg:dp}.

\paragraph{Associated random walk.} It remains to show that $\E$ is also 
the expectation of $\Y \in \cY$ support of the following random walk,
defined informally in the main text.
Formally, we define the random sequence ${(w_t)}_t$~as
\begin{equation}
    w_0 = N,\quad
    \forall\,t>0,\,\forall\,i\in[N],\,\forall\,j \in \cP_i,\,\quad
    \PP[w_t=j | w_{t-1} = i] = q_{i, j}.
\end{equation}
We set~$y_{i, j} \triangleq
    \bm{1}\{\exists\,t > 0\text{ s.t. } w_{t-1} = i, w_t = j\}$
where $\bm 1$ is the characteristic function of an event, thereby defining
a random variable $\Y \in \cY$, with distribution $\cD$.
We leave implicit the dependency of $\PP$ in $\btheta$ and $\Omega$.
As the depth of $w_t$ (number of edges to connect to the root node) is
stricly decreasing with $t$, ${(w_t)}_t$ reaches node $1$ in finite time
 with probability one and is constant after this event. We introduce
 the random variables ${(\bar y_j)}_j$, defined for all $j \in [N]$ as
 \begin{equation}
    \bar y_j \triangleq \bm{1}\{\exists\,t \ge 0, w_{t} = j\} = \sum_{i \in \cC_j} y_{i, j}
    \text{ if $j \neq N$, $0$ otherwise.}
 \end{equation}
 
 By definition,
 using the fact that $\PP[w_t=j | w_{t-1} = i]$ is independent of $t$ (Markov
 property), for all $i \in \cC_j$ and for all $j \in [N-1]$, we have
 \begin{align}
     \PP[y_{i,j} = 1] &= \EE[y_{i,j}] = \PP[\exists\,t > 0, w_{t-1} = i]
      \PP[w_t=j | w_{t-1} = i] = \EE[\bar y_i] q_{i, j}.
 \end{align}
Linearity of the expectation then provides
\begin{equation}
    \EE[\bar y_j] = \sum_{i \in \cC_j} \EE[y_{i, j}],
\end{equation}
with initialization $\EE[\bar y_N] = 1$. We recover the same two-step
recursion as the one defining $\E$ and $\bar \e$, with the same initialization. Hence
the probabilistic interpretation of the gradient, where the expectation
is taken with respect to the distribution $\cD$ of $\Y$:
\begin{equation}
\E = \EE_{\btheta,\Omega}[\Y] \quad \text{and} \quad \bar \e = \EE_{\btheta,\Omega}[\bar \y].
\end{equation}

\begin{figure}
    \begin{widepage}
    \colorbox{white}{
    \begin{minipage}[t]{0.48\textwidth}
       \begin{algorithm}[H]
   \begin{algorithmic}
     \Input Edge weights $\btheta \in \RR^{N \times N}$
     \State \mygray{$v_1 \gets 0,\quad\bar e_N \gets 1, \qquad \Q, \E \gets \mathbf{0} \in \RR^{N \times N}$}
     \For{$i \in [2,\dots, N]$} \qquad\Comment{Topological order}
     \State $v_i \gets \underset{j \in \cP_i}{\maxOmega}~ \theta_{i,j} + v_j$
     \State ${(\q_{i, j})}_{j \in \cP_i} \gets \nabla\underset{j \in
 \cP_i}{\maxOmega}~ \theta_{i,j} + v_j$
     \EndFor
     \For{$j \in [N - 1, \dots,1]$} \quad\Comment{Reverse topological order}
     \State $\forall\,i \in \cC_j,\: e_{i,j} \gets e_{i,j} \bar e_i$,\qquad
     $\bar e_j \gets \sum_{i \in \cC_j} e_{i,j}$
     \EndFor
     \State \Return $\dpOmega(\btheta) = v_N$, $\nabla\dpOmega(\btheta) = \E \in \RR^{N \times N}$\\
     \hspace{1.4cm}Intermediate computation for Algorithm \ref{alg:hessian_dp} \\
     \hspace{1.4cm}$\bar \e \triangleq [\bar e]_{i=1}^N \in \RR^N$, $\Q \in \RR^{N \times N}$
     \vspace{.07cm}
   \end{algorithmic}
   \caption{Compute $\dpOmega(\btheta)$ and $\nabla \dpOmega(\btheta)$}\label{alg:dp}
   \end{algorithm}
   \end{minipage}}\hfill
   \begin{minipage}[t]{0.48\textwidth}
    \newcounter{algeq}
    \let\oldthealgeq\thealgeq
    \renewcommand{\thealgeq}{A\oldthealgeq}
   \begin{algorithm}[H]
   \begin{algorithmic}
     \Input Edge weights and perturbation $\btheta, \Z \in \RR^{N \times N}$
     \State Call Algorithm \ref{alg:dp} with input $\btheta$ to get $\bar \e$ and $\Q$
     \State \mygray{$\dot v_1 \gets 0;\qquad\dot{\bar{e_N}} \gets 0,
     \qquad \dot \Q, \dot \E \gets \mathbf{0} \in \RR^{N \times N}$}
     \For{$i \in [2, \dots,N]$} \quad\Comment{Topological order}
     \State $\dot v_i \gets \sum_{j \in \cP_i} q_{i,j} (z_{i, j} + \dot v_j)$
     \hfill \refstepcounter{algeq}(\thealgeq)\label{eq:directional}
     \State ${(\dot \q_{i, j})}_{j \in \cP_i} \gets \J_\Omega\big((\q_{i, j})_{j \in \cP_i}\big)
     (z_{i,j} + \dot v_j)_{j \in \cP_i}$
     \hfill\refstepcounter{algeq}(\thealgeq)\label{eq:hessian}
     \EndFor
     \For{$j \in [N - 1,\dots, 1]$} \quad \Comment{Reverse topological order}
     \State $\forall\,i \in \cC_j,\: \dot e_{i,j} \gets \dot q_{i,j} \bar e_i + q_{i,j} \dot{\bar{e_i}}$
     \hfill \refstepcounter{algeq}(\thealgeq)\label{eq:hessian_backprop}
     \State $\dot{\bar{e_j}} \gets \sum_{i \in \cC_j} \dot e_{i,j}$
     \EndFor
     \State \Return $\langle\nabla \dpOmega(\btheta), \Z\rangle = \dot v_N$ \\
     \hspace{1.5cm}$\nabla^2 \dpOmega(\btheta)\Z = \dot \E \in \RR^{N \times N}$
   \end{algorithmic}
   \caption{Compute $\langle\nabla\dpOmega(\btheta), \Z\rangle$ and $\nabla^2 \dpOmega(\btheta)\Z$}\label{alg:hessian_dp}
   \end{algorithm}
   \end{minipage}
    \end{widepage}
\end{figure}

\subsection[Computation of directional derivative]{Computation
of the directional derivative $\langle \nabla \dpOmega(\btheta), \Z \rangle$}
\label{appendix:directional_derivative}

The derivations of the following two sections allows to write Algorithm~\ref{alg:hessian_dp}. Let $\dot v_i \triangleq \langle\nabla v_i(\btheta), \Z\rangle$,
where $v_i(\btheta)$ is defined in \eqref{eq:smoothed_recursion}.
Since $v_i$ only directly depends on $v_j + \theta_{i,j}$ for
$j \in \cP_i$, a straighforward differentiation of 
$\langle\nabla v_i(\btheta), \Z\rangle$ gives
\begin{equation}
\dot v_i = \sum_{j \in \cP_i} \partialfrac{v_i}{v_j}
~ \left(\dot v_j + z_{i, j}\right).
\end{equation}
Recall that $\partialfrac{v_i}{v_j} = q_{i,j}$ and has already been
obtained when computing $\nabla \dpOmega(\btheta)$. Hence equation~\eqref{eq:directional}, reproduced here:
\begin{equation}
\forall\,i \in [2,\dots, N]:\qquad\dot v_i = \sum_{j \in \cP_i} q_{i, j} (\dot v_j + z_{i, j}).
\label{eq:v_dot_i}
\end{equation}
This recursion can be computed in topological order, starting from $\dot
v_1 = 0$ to finish at $\dot v_N = \langle \nabla \dpOmega(\btheta), \Z \rangle$.

\subsection[Computation of Hessian product]{Computation of the Hessian-vector product $\nabla^2 \dpOmega(\btheta) \Z$}
\label{appendix:hessian_vector_product}

For convenience, let us define $\nabla^2 \dpOmega(\btheta) \Z
\triangleq \dot \E$.
For $(i,j) \notin \cE$, we evidently have $\dot e_{i,j} = 0$.
For $(i,j) \in \cE$, since $\theta_{i, j}$ influences only
$v_i$ and $\dot v_i$, we obtain
\begin{equation}
\dot e_{i, j} = \partialfrac{\dot v_N}{\theta_{i, j}}
= \mymagenta{\partialfrac{\dot v_N}{v_i}} 
\mygreen{\partialfrac{v_i}{\theta_{i,j}}}
+ \myblue{\partialfrac{\dot v_N}{\dot v_i}} \myyellow{\partialfrac{\dot
v_i}{\theta_{i, j}}}.
\end{equation}
We will now show how to derive each of the right-hand side terms in turn.
We already know that $\mygreen{\partialfrac{v_i}{\theta_{i,j}}} = q_{i,j}$.
We also have $\myblue{\partialfrac{\dot v_N}{\dot v_i}} = u_i$.
Indeed, observe that
$\dot v_j$ only directly influences $\dot v_i$ for~$i \in \cC_i$. Therefore,
we have
\begin{equation}
\partialfrac{\dot v_N}{\dot v_j} = \sum_{i \in \cC_j}
\myblue{\partialfrac{\dot v_N}{\dot v_i}} q_{i, j} \quad ~ \forall j \in [N-1]
\label{eq:dot_v_N_dot_v_j}
\end{equation}
and $\partialfrac{\dot v_N}{\dot v_1} = 1$. 
Comparing \eqref{eq:v_N_v_j} and \eqref{eq:dot_v_N_dot_v_j}, 
we see that
${(\partialfrac{\dot v_N}{\dot v_i})}_i$
follows the same recursion as ${(\partialfrac{v_N}{v_i})}_i$.
Since $\partialfrac{\dot v_N}{\dot v_n} = \partialfrac{v_N}{v_n}$, both
sequences are equal:
\begin{equation}
\myblue{\partialfrac{\dot v_N}{\dot v_i}} = \partialfrac{v_N}{v_i} = e_i.
\end{equation}
Next, we derive $\myyellow{\partialfrac{\dot
v_i}{\theta_{i,j}}}$.
Since, for $j \in \cP_i$, $\dot v_j + z_{i, j}$ does not depend on $\theta_{i,
j}$, differentiating \eqref{eq:v_dot_i} \wrt $\theta_{i,j}$, we obtain
\begin{align}
    \myyellow{\partialfrac{\dot v_i}{\theta_{i, j}}} &=
\sum_{k \in \cP_i} \partialfrac{q_{i, j}}{\theta_{i, j}} (\dot v_k + z_{i, k})
\\
&= \sum_{k \in \cP_i} \frac{\partial^2 v_i}{\partial \theta_{i, j }
\partial \theta_{i, k}} (\dot v_k + z_{i, k})
\triangleq \dot q_{i,j}.
\end{align}
This can be conveniently rewritten in a vectorial form as
\begin{equation}
\dot \q_i = \nabla^2 \maxOmega(\btheta_i + \v) ~ (\z_i + \dot \v) 
= \J_\Omega(\q_i) ~ (\z_i + \dot \v),
\end{equation}
where we have defined $\dot \v \triangleq (\dot v_1, \dots, \dot v_N)$ and where
we have used the function $\J_\Omega$ defined in
\S\ref{appendix:omega_examples}, that conveniently computes the Hessian of
$\maxOmega$ from its gradient. The Hessian has this form for both negentropy and
$\ell_2^2$ regularizations. In a practical implementation, we only need to compute
the coordinates $(i, j)$ of $\dot \Q$,
for $j \in \cP_i$. Namely, as specified in~\eqref{eq:hessian},
\begin{equation}
    (\dot \q_{i, j})_{j \in \cP_i} \gets \J_\Omega\big((\q_{i, j})_{j \in \cP_i}\big)
    (z_{i,j} + \dot v_j)_{j \in \cP_i}.
\end{equation}
Finally, we derive $\mymagenta{\partialfrac{\dot v_N}{v_i}}$.
Since $v_j$ influences only $v_i$ and $\dot v_i$ for $i \in \cC_j$,
the chain rule gives 
\begin{equation}
\mymagenta{\partialfrac{\dot v_N}{v_i}} 
= \sum_{j \in \cC_i}
  \partialfrac{\dot v_N}{v_j} \partialfrac{v_j}{v_i}
  + \partialfrac{\dot v_N}{\dot v_j} \partialfrac{\dot v_j}{v_i}
  = \sum_{j \in \cC_j} \dot e_{i, j}
  \triangleq \dot{\bar{e_i}}.
\end{equation}
Combining the above, for any $j \in [N-1]$, we obtain the following two-step
recursion~\eqref{eq:hessian_backprop}, reproduced here:

\begin{equation}
    \forall\,i \in \cC_j,\quad\dot e_{i,j} = \dot q_{i,j} e_i + q_{i,j} \dot{\bar{e_i}}
\quad \text{and} \quad
\dot{\bar{e_j}} = \sum_{i \in \cC_j} \dot e_{i,j}.
\end{equation}
Similarly to the computation of $\nabla \dpOmega(\btheta)$, our algorithm
computes this recursion in reverse topological order over the graph $G$, yielding
$\nabla^2\dpOmega(\btheta)\Z = \dot \E$.

\section{Examples of algorithm instantiations}

We provide the explicit forms of $\maxOmega$ and its derivative for
the negentropy and $\ell_2^2$ regularizations. Then, we provide details and
pseudo-code for the two instances of differentiable dynamic programming presented in
\S\ref{sec:examples}.

\subsection{Examples of $\maxOmega$}\label{appendix:omega_examples}

\textbf{Negative entropy.} When $\Omega(\q) = \gamma \sum_{i=1}^D q_i \log
q_i$, where $\gamma > 0$ (smaller is less regularized), we obtain
\begin{align}
    \maxOmega(\x) &= \gamma \log\left(\sum_{i=1}^D \exp(x_i / \gamma)\right) \\
    \nabla \maxOmega(\x) &= \exp(\x/\gamma) \Big/ \sum_{i=1}^D \exp(x_i /
    \gamma) \\
    \nabla^2 \maxOmega(\x) &= \J_\Omega(\nabla \maxOmega(\x)),
\end{align}
where $\J_\Omega(\q) \triangleq (\diag(\q) - \q \q^\top) / \gamma$.  Note that
$\nabla \maxOmega(\x)$ recovers the usual ``softmax'' with temperature 
$\gamma =1$. For a proof of the expression of $\maxOmega$, see, \textit{e.g.,}
\citep[Example 3.25]{boyd2004convex}.

\paragraph{Squared $\ell_2$ norm.} When $\Omega(\x) = \frac{\gamma}{2}
\|\x\|^2_2$ with $\gamma > 0$, we obtain the following expressions
\begin{align}
    \maxOmega(\x) &= \langle \q^\star, \x \rangle - \frac{\gamma}{2}
    \|\q^\star\|^2_2\\
    \nabla \maxOmega(\x) &= \argmin_{\q \in \Simplex^D} \|\q - \x / \gamma\|^2_2 
    = \q^\star \\
    \nabla^2 \maxOmega(\x) &= \J_\Omega(\nabla \maxOmega(\x)),
\end{align}
where $\J_\Omega(\q) \triangleq (\diag(\s) - \s \s^\top / \|\s\|_1) / \gamma $
and $\s \in \{0,1\}^D$ is a vector that indicates the support of $\q$. Note that
$\nabla \maxOmega(\x)$ is precisely the Euclidean projection onto the simplex of
$\x / \gamma$ and can be computed exactly in worst-case $\bigO(D \log D)$ time
using the algorithm of \citep{michelot} or in expected $\bigO(D)$ time using the
randomized pivot algorithm of \citep{duchi}. It can be efficiently 
performed on Nvidia GPUs since recently. An important benefit of the
squared $\ell_2$ norm, compared to the negative entropy, is that $\nabla
\maxOmega(\x)$ tends to be sparse.  This is useful, among other things, to
define sparse attention mechanisms \citep{sparsemax,niculae_blondel_2017}.

\subsection{Sequence prediction with the smoothed Viterbi algorithm}%
\label{appendix:viterbi}

\textbf{Computational graph.} As illustrated in \S\ref{sec:examples}, the DAG
contains a start node, $S$ nodes for each time step and end node. Therefore
$|\cV|=N=TS+2$. Only nodes from consecutive time steps are connected to each
other. Taking into account the start and end nodes, the total number of edges is
therefore $|\cE| = (T-1)S^2 + 2S$.

\textbf{Representation.} We follow the notation of \S\ref{sec:examples}, \textit{i.e.}
we represent $\Y$ and $\btheta$ as $T \times S \times S$ tensors (we can safely
ignore the edges connected to the end node since their value is $0$). We
represent $\Y$ as a binary tensor such that $y_{t,i,j}=1$ if $\Y$ is in states
$i$ and $j$ in time steps $t$ and $t-1$, and $y_{t,i,j}=0$ otherwise. Likewise,
we represent the potentials $\btheta$ as a real tensor such that
$\theta_{t,i,j}$ contains the potential of transitioning from state $j$ to state
$i$ on time $t$.

\textbf{Algorithms.} Applying recursion \eqref{eq:smoothed_recursion} to this
specific DAG, we obtain a smoothed version of the Viterbi algorithm. Let
$v_{t,i}$ be the score of being in state $i$ up to time $t$. We can
rewrite the smoothed Bellman recursion as
\begin{equation}
v_{t,i}(\btheta) \triangleq \underset{j \in [S]}{\maxOmega} ~
v_{t-1,j}(\btheta) + \theta_{t,i,j} 
= \maxOmega(\v_{t-1}(\btheta) + \btheta_{t,i}).
\end{equation}
The value $\viterbiOmega(\btheta) \triangleq \maxOmega(\v_T(\btheta))$ can be
computed in topological order, starting from $\v_0(\btheta)$.  The total
computational cost is $\bigO(TS^2)$. Using the computations of \S\ref{sec:differentiation}
and \S\ref{sec:hessian} to this
specific DAG, we can compute $\nabla \viterbiOmega(\btheta)$, $\langle \nabla
\viterbiOmega(\btheta), \Z \rangle$ and $\nabla^2 \viterbiOmega(\btheta) \Z$
with the same complexity.  The procedures are summarized in Algorithm
\ref{algo:Viterbi_grad} and Algorithm \ref{algo:Viterbi_Hess}, respectively.
From Proposition \ref{proposition:dpOmega} property
\ref{property:dpOmega_convex}, $\viterbiOmega(\btheta)$ is a convex function for
any $\Omega$.

\begin{figure}[t]
\begin{widepage}
\colorbox{white}{
\begin{minipage}[t]{0.48\linewidth}
    \begin{algorithm}[H]
\begin{algorithmic}
  \Input Potential scores $\btheta \in \RR^{T \times S \times S}$
  \State \Comment{Forward pass}
  \State \mygray{$\v_{0} = \zeros_S$}
  \For{$t \in [1,\dots, T], i \in [S]$}
  \State $v_{t,i} = \maxOmega(\btheta_{t,i} + \v_{t-1})$
  \State $\q_{t,i} = \nabla\maxOmega(\btheta_{t,i} + \v_{t-1})$
  \EndFor
  \State $v_{T+1, 1} = \maxOmega(\v_T); \quad  \q_{T+1,1} = \nabla \maxOmega(\v_T)$
  \State \Comment{Backward pass}
  \State \mygray{$\u_{T+1} = (1, 0, \dots, 0) \in \RR^S$}
  \For{$t\in [T,\dots, 0], j \in [S]$} 
  \State $\e_{t, \cdot, j} = \q_{t + 1, \cdot, j} \circ \u_{t + 1};\quad 
  u_{t, j} = \langle \e_{t, \cdot, j}, \ones_S \rangle$
  \EndFor
  \State \Return $\viterbiOmega(\btheta) = v_{T+1, 1}$\\
  \hspace{1.3cm} $\nabla \viterbiOmega(\btheta) = (e_{t-1, i, j})_{t=1,i,j=1}^{T,S,S}$ \\
  \hspace{1.3cm} Intermediary computations for Alg.~\ref{algo:Viterbi_Hess}:

  \hspace{1.3cm} $\Q \triangleq (q)_{t=1,i,j=1}^{T + 1, S, S},
  \U \triangleq (u)_{t=1,j=1}^{T + 1, S}$
\end{algorithmic}
\caption{Compute $\viterbiOmega(\btheta)$ and $\nabla \viterbiOmega(\btheta)$}%
\label{algo:Viterbi_grad}
\end{algorithm}
\end{minipage}}\hfill
\begin{minipage}[t]{0.48\linewidth}
\begin{algorithm}[H]
\begin{algorithmic}
  \Input $\Z \in \RR^{T \times S \times S},
   \btheta \in \RR^{T \times S \times S}$
  \State Call Alg.~\ref{algo:Viterbi_grad} with input $\btheta$ to get $\U$, $\Q$
  \State \Comment{Forward pass}
  \State \mygray{$\dot \v_{0} = \zeros_S$}
  \For{$t \in [1,\dots, T], i \in [S]$}
  \State $\dot v_{t,i} =
  \langle \q_{t, i}, \z_{t, i} + \dot \v_{t-1} \rangle$
  \State $\dot \q_{t,i} = \J_\Omega(\q_{t,i}) ~ (\z_t + \dot \v_{t-1})$
  \EndFor
  \State $\dot v_{T+1,1} = \langle \q_{T+1, 1}, \dot \v_T \rangle;\quad \dot \q_{T+1, 1} = \J_\Omega(\dot \q_{T + 1, 1}) ~ \dot \v_T$
  \State \Comment{Backward pass}
  \State \mygray{$\dot \u_{T+1} = \zeros_S;\quad\dot \Q_{T+1} = \zeros_{S \times S}$}
  \For{$t\in [T,\dots, 0], j \in [S]$} 
  \State $\dot \e_{t, \cdot, j} = \q_{t + 1, \cdot, j} \circ \u_{t + 1}
  + \dot \q_{t + 1, \cdot, j} \circ \dot{\u}_{t + 1} $ 
  \State $\dot{u}_{t, j} = \langle \dot \e_{t, \cdot, j}, \ones_S \rangle$
  \EndFor
  \State \Return $\langle\viterbiOmega(\btheta),\Z\rangle = \dot v_{T+1}$ \\
  \hspace{1.3cm} $\nabla^2 \viterbiOmega(\btheta)\Z = (\dot e_{t-1,i, j})_{t=1,i,j=1}^{T,S,S}$
\end{algorithmic}
\caption{Compute $\langle\nabla \viterbiOmega(\btheta),\Z\rangle$
and $\nabla^2 \viterbiOmega(\btheta)\Z$}%
\label{algo:Viterbi_Hess}
\end{algorithm}
\end{minipage}
\end{widepage}
\end{figure}
\subsection{Monotonic aligment prediction with the smoothed DTW}
\label{appendix:dtw}

\textbf{Computational graph.} As illustrated in \S\ref{sec:examples}, the DAG
contains a start node and $N_A N_B$ nodes. Therefore, $|\cV|=N=N_A N_B + 1$.
Due to the monotonic constraint, each node may only be connected with at most
$3$ other nodes.
The cardinality of $\cY$
is the $\text{delannoy}(N_A{-}1,N_B{-}1)$ number \citep{sulanke2003,banderier2005}.
That number grows exponentially with~$N_A$ and~$N_B$.  

\textbf{Representation.} We follow the notation of \S\ref{sec:examples}, \textit{i.e.}
we represent $\Y$ and $\btheta$ as $N_A \times N_B$ matrices.  We represent $\Y$
as a binary matrix such that $y_{i,j}=1$ if $\a_i$ is aligned with $\b_j$, and
$y_{i,j}=0$ otherwise. Likewise, we represent $\btheta$ as a real matrix such
that $\theta_{i,j}$ is a measure of ``discrepancy'' between $\a_i$ and $\b_j$.

\textbf{Algorithms.} Following the DTW literature \citep{Sakoe78}, we seek
an alignment with \textit{minimal} cost. For that reason, we introduce
the smoothed min operator, its gradient and its Hessian as follows
\begin{align}
    \minOmega(\x) &\triangleq -\maxOmega(-\x) \\
    \nabla \minOmega(\x) &= \nabla \maxOmega(-\x) \\
    \nabla^2 \minOmega(\x) &= -\nabla^2 \maxOmega(-\x) \\
                           &= - \J_\Omega(\nabla \maxOmega(-\x)) \\
                           &= -\J_\Omega(\nabla \minOmega(\x)).
\end{align}
Applying \eqref{eq:smoothed_recursion} to the DTW DAG gives rise to a smoothed
version of the algorithm.  Let $v_{i,j}(\btheta)$ be the alignment cost up to
cell $(i,j)$. Then the smoothed DTW recursion is
\begin{equation}
    v_{i,j}(\btheta) = \theta_{i,j} + \minOmega(v_{i, j-1}(\btheta), v_{i-1,
    j-1}(\btheta), v_{i-1, j}(\btheta))
\end{equation}
The value $\dtwOmega(\btheta) \triangleq \v_{N_A,N_B}(\btheta)$ can be computed
in $\bigO(N_A N_B)$ time.  Applying the derivations of \S\ref{sec:differentiation}
and \S\ref{sec:hessian} to this specific DAG, we can compute
$\nabla \dtwOmega(\btheta)$, $\langle \nabla \dtwOmega(\btheta), \Z
\rangle$ and $\nabla^2 \dtwOmega(\btheta) \Z$ with the same complexity.  The
procedures, with appropriate handling of the edge cases, are summarized in
Algorithm \ref{algo:dtw_grad} and~\ref{algo:dtw_Hess}, respectively.

Note that when $\Omega$ is the negative entropy, $\dtwOmega(\btheta)$ is known
as soft-DTW \citep{soft_dtw}. While the DP computation of $\dtwOmega(\btheta)$
and of its gradient were already known, the generalization to any strongly
convex $\Omega$ and the computation of $\nabla^2 \dtwOmega(\btheta) \Z$ are
new. From Proposition \ref{proposition:dpOmega} property
\ref{property:dpOmega_convex}, $\dtwOmega(\btheta)$ is a \textit{concave}
function of the discrepancy matrix $\btheta$ for
any $\Omega$. With respect to time-series, $\dtwOmega$ is neither convex nor concave.

\begin{figure}[t]
\begin{widepage}
\colorbox{white}{
\begin{minipage}[t]{0.48\linewidth}
    \begin{algorithm}[H]
\begin{algorithmic}
  \Input Distance matrix $\btheta \in \RR^{N_A \times N_B}$
  \State \Comment{Forward pass}
  \State \mygray{$v_{0,0} = 0$; $v_{i,0} = v_{0,j} = \infty$, $i \in [N_A], j \in
  [N_B]$}
  \For{$i \in [1, \dots, N_A], j \in [1, \dots, N_B]$}
  \State $v_{i,j} = d_{i,j} + \minOmega(\myred{v_{i, j-1}}, \myblue{v_{i-1,
  j-1}}, \mygreen{v_{i-1, j}})$
  \State $\q_{i, j} = \nabla \minOmega(\myred{v_{i, j-1}}, \myblue{v_{i-1,
  j-1}}, \mygreen{v_{i-1, j}})
  \in \RR^3$
  \EndFor
  \State \Comment{Backward pass}
  \State \mygray{$\q_{i,N_B+1} = \q_{N_A+1,j} = \zeros_3$, $i \in [N_A], j \in
  [N_B]$}
  \State \mygray{$e_{i,N_B+1} = e_{N_A+1,j} = 0$, $i \in [N_A], j \in [N_B]$}
  \State \mygray{$\q_{N_A+1, N_B+1} = (0, 1, 0)$; $e_{N_A+1,N_B+1} = 1$}
  \For{$j \in [N_B, \dots, 1], i \in [N_A, \dots, 1]$}
  \State $e_{i,j} = \myred{q_{i, j+1, 1} ~ e_{i,j+1}} +
  \myblue{q_{i+1, j+1, 2} ~ e_{i+1,j+1}} +$
  \State \hspace{1.3cm} $\mygreen{q_{i+1, j, 3} ~ e_{i+1,j}}$
  \EndFor
  \State \Return $\dtwOmega(\btheta) = v_{N_A, N_B}$ \\
  \hspace{1.3cm} $\nabla \dtwOmega(\btheta) = (e)_{i,j=1}^{N_A,N_B}$ \\
  \hspace{1.3cm} Intermediate computations for Algo. \ref{algo:dtw_Hess}:

  \hspace{0.8cm} $\Q \triangleq (q)_{i,j,k=1}^{N_A+1,N_B+1,3}$;
  $\E \triangleq (e)_{i,j=1}^{N_A+1,N_B+1}$
\end{algorithmic}
\caption{\small Compute $\dtwOmega(\btheta)$ and $\nabla \dtwOmega(\btheta)$}
\label{algo:dtw_grad}
\end{algorithm}
\end{minipage}}\hfill
\begin{minipage}[t]{0.48\linewidth}
  \begin{algorithm}[H]
    \begin{algorithmic}
        \Input $\btheta \in \RR^{N_A \times N_B}, \Z \in \RR^{N_A \times N_B}$
        \State Call Algo. \ref{algo:dtw_grad} with input $\btheta$ to retrieve $\Q$ and $\E$
      \State \Comment{Forward pass}
      \State \mygray{$\dot v_{i,0} = \dot v_{0,j} = 0$, $i \in
          [0,\dots,N_A], j \in [N_B]$}
      \For{$i \in [1, \dots, N_B], j \in [1, \dots, N_A]$}
      \State $\dot v_{i,j} = z_{i,j} +
      \myred{q_{i, j, 1} ~ \dot v_{i, j-1}} + \myblue{q_{i, j, 2} ~ \dot v_{i-1,
      j-1}} +$
      \State \hspace{1.3cm} $\mygreen{q_{i, j, 3} ~ \dot v_{i-1, j}}$
      \State $\dot \q_{i, j} = -\J_\Omega(\q_{i, j}) ~ (\myred{\dot v_{i,
      j-1}}, \myblue{\dot v_{i-1, j-1}}, \mygreen{\dot v_{i-1, j}}) \in \RR^3$
      \EndFor
      \State \Comment{Backward pass}
      \State \mygray{$\dot \q_{i,N_B+1} = \dot \q_{N_A+1,j} = \mathbf{0}_3$, $i \in
      [0,\dots,N_A], j \in [N_B]$}
      \State \mygray{$\dot e_{i,N_B+1} = \dot e_{N_A+1,j} = 0$, $i \in
          [0,\dots,N_A],
      j \in [N_B]$}
      \For{$j \in [N_B, \dots, 1], i \in [N_A, \dots, 1]$}

      \State $\dot e_{i,j} =
      \myred{\dot q_{i, j+1, 1} ~      e_{i,j+1}} +
      \myred{q_{i, j+1, 1} ~ \dot e_{i,j+1}} + $
      \State \hspace{1.3cm} $\myblue{\dot q_{i+1, j+1, 2} ~      e_{i+1,j+1}} +
      \myblue{q_{i+1, j+1, 2} ~ \dot e_{i+1,j+1}} +$
      \State \hspace{1.3cm} $\mygreen{\dot q_{i+1, j, 3} ~ e_{i+1,j}} +
      \mygreen{q_{i+1, j, 3} ~ \dot e_{i+1,j}}$
      \EndFor
      \State \Return $\langle \nabla \dtwOmega(\btheta), \Z \rangle = \dot
      v_{N_A,N_B}$ \\
              \hspace{1.3cm} $\nabla^2 \dtwOmega(\btheta) ~ \Z = (\dot
              e)_{i,j=1}^{N_A,N_B}$
    \end{algorithmic}
    \caption{\small Compute $\langle \nabla \dtwOmega(\btheta), \Z \rangle$,
    $\nabla^2 \dtwOmega(\btheta) ~ \Z$}
    \label{algo:dtw_Hess}
  \end{algorithm}
\end{minipage}
\end{widepage}    
\end{figure}

\section{Experimental details and further results}

We finally provide details on the architecture used in experiments, with additionnals figures.

\begin{figure}[t]
    \centering
    \makebox[\textwidth][c]{
    \hspace{.4cm}    
    \includegraphics[width=1.1
    \textwidth]{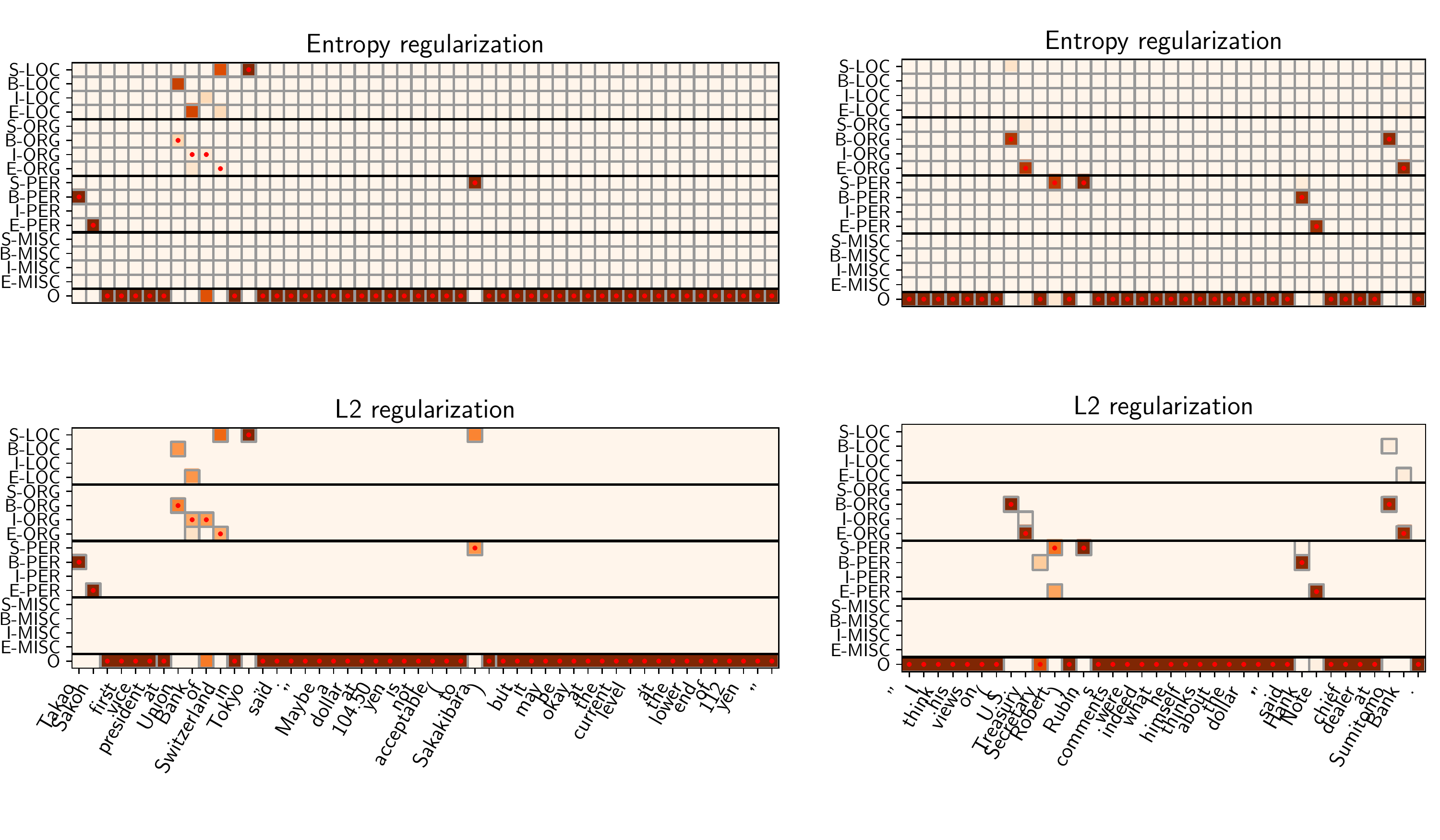}%
    }
    \caption{Test predictions from the entropy and $\ell_2^2$
        regularized named entity recognition (NER) models. Red dots indicate
        ground truth. When using $\ell_2^2$ regularization, model predictions are
    sparse (grey borders indicates non-zero cells).
    They are thus easier to introspect for ambiguities,
    as we can list a finite number of possible outputs.}\label{fig:ner}
\end{figure}

\subsection{Named entity recognition (section \S\ref{sec:ner})}\label{appendix:ner}

Our model extracts word embedding from a $300$-dimensional lookup table concatenated with
a $50$-dimensional character embedding. This character embedding corresponds to the concatenation
of the last hidden unit of a bi-directional character LSTM, as in~\citet{lample_2016}.
Character embedding size is set to $50$. A word LSTM then produces sentence-aware features
for each word. This LSTM is bi-directional with $100$-dimensional hidden units per direction.
The final features $\X$ used to build the potential tensor $\btheta$ are thus $200$-dimensional.
Note that, in contrast with~\citet{lample_2016}:
\begin{itemize}[topsep=0pt,itemsep=0pt,parsep=0pt]
    \item The look-up table is initialized with 300-dimensional 
        embeddings from
        \textit{FastText} \citep{joulin2016fasttext}, trained on Wikipedia corpus.
    \item We do not pad letters prior to feeding the character LSTM as it is not principled.
    \item We do not train the unknown word embedding as we found it had no effect.
\end{itemize}
We convert tags to the IOBES (Inside-Outside-Begin-End-Stop) scheme to build a richer $\viterbiOmega$ model than if we used 
the simpler IOB (Inside-Outside-Begin) scheme, that has a lower number of tags. We performed a small grid-search
 to select the step-size and batch-size used for optimization: $s \in \{0.005, 0.01, 0.02\}$, $b \in \{ 8, 32, 128\}$.
For each language and each loss, we select the highest-scoring model on the validation set, and report the 
test score.

The model is strongly subject to overfitting using the convex surrogate loss and the log likelihood. We have
to use a small batch size ($b = 8$) and vanilla SGD with large step size $(s = 0.01)$ to avoid this overfitting issue.
For all losses, accelerated stochastic optimizers have all lower generalization
performance than SGD, as also noticed
in~\citep{lample_2016} when using the classical negative log-likelihood as a loss.

\paragraph{Visualization.} The models using $\ell_2^2$ regularization perform
nearly on par with the ones using negentropy, as 
demonstrated in Table~\ref{table:ner_results}. On the other hand, $\ell_2^2$
regularization leads to 
tag probability vectors that are sparse and hence easier to parse. They allow to detect ambiguities more
easily. We display a few tagged English sequences in Figure~\ref{fig:ner}. The
model using $\ell_2^2$ regularization correctly identifies
an ambiguous entity (\textit{Union Bank of Switzerland}) and can be used to propose two tag sequences:
 \textit{(B-ORG, I-ORG, I-ORG, E-ORG)}
or \textit{(B-ORG, E-ORG, O, S-LOC)}. Probabilities of every tag sequence can be computed using the matrix
$\Q$, as described in \S\ref{sec:differentiation} --- this remains tractable as long as the matrix $\Q$ is 
\textit{sparse enough},
so that the number of non-zero probabilities sequence remains low.
On the other hand, the model using negentropy regularization
never assign a zero probability to any tag sequence
--- it is therefore not tractable to provide the user with a small set of interesting sequences.

\begin{figure}
    \centering
    \includegraphics[width=.8\textwidth]{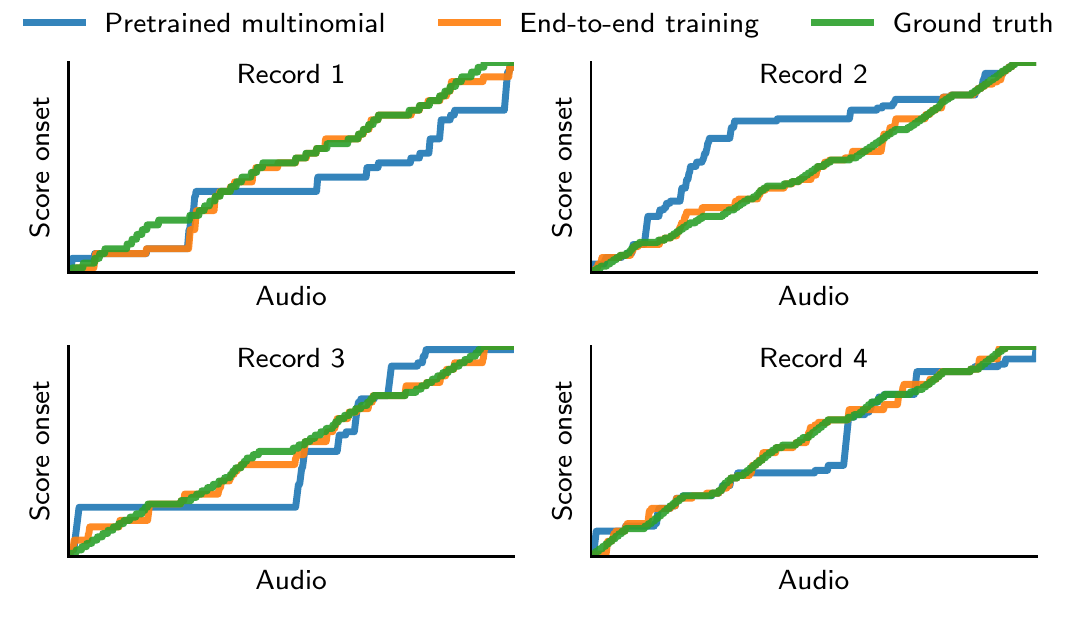}
    \caption{Alignment maps between score onsets and audio frames on test data from the Bach10 dataset.
    Our end-to-end trained model qualitatively performs better than the baseline model.}\label{fig:onsets}
\end{figure}

\subsection{Supervised audio-to-score transcription (section \S\ref{sec:audio})}\label{app:audio}

Audio sequences, sampled at $22.05\,\text{kHz}$, are split into frames of $512$ samples. We extract
the following features from these sequences: energy, spectral centroid, spectral bandwidth, and the 5 first
MFCC features. All features are centered around the median and normalized. 
The $\nabla \dtwOmega$ layer is written in
\textit{Cython}\footnote{\url{http://cython.org/}}, and hence run on CPU. This technical
choice was suggested by the fact that we have to write explicit loops to specify the topological and
reverse topological pass over the DTW computation graph (see Algorithm \ref{algo:dtw_grad}).
However, it is possible to use only contiguous vector operations and thus take advantage of GPU computations --- this is left
for future work. We use \textit{SciPy}'s\footnote{\url{http://scipy.org/}} LBFGS-B solver to perform end-to-end training and multinomial
regression. We use a $\ell_2^2$ regularization on the weight $\W$,:
we selected it using a grid search over $\{10^{-5}, 10^{-4}, \dots, 1\}$ and selected $10^{-3}$.

\paragraph{Further vizualisation.}In Figure~\ref{fig:onsets}, we display the alignment maps we obtained using our algorithm and using the
baseline multinomial model followed by a hard-DTW alignment computation. These alignment maps correspond to the
predicted onsets of keys. Our model (in orange) performs visibly better in
predicting onsets.

\subsection{Structured and sparse attention (section \S\ref{sec:structured_attention})}\label{appendix:attention}

We use \textit{OpenNMT-py} library\footnote{\url{http://opennmt.net/}} to fit our structured attention model.
 Model architecture and optimization details are as follow:
\begin{itemize}[topsep=0pt,itemsep=0pt,parsep=0pt]
    \item We use a bidirectional LSTM encoder and decoder, with 500 units in each direction
    and a depth of 2 layers .
    \item The decoder is fed with the input representation as in \citet{luong_effective_2015}.
    \item SGD training with $s = 1$ learning rate,
     decaying from epoch 8 to epoch 15 with rate $0.65$, batch size of size $256$.
    \item Training sentence of lengths superior to $50$ are ignored, and
    translated sentence are forced to a length inferior to $100$.
    \item The temperature parameter is set to $\gamma = 2$ for entropy, and $\gamma = 10$ for $\ell_2^2$.
    Performance is not affected much by this parameter, provided that it is not
    set too low in the $\ell_2^2$ case --- with a too small $\gamma$,
    $\viterbiOmega$ reduces to unregularized
    MAP estimation and $\nabla\viterbiOmega$ has zero derivatives.
\end{itemize}
We use a $1$-million sentence subject of WMT14 English-to-French corpus, available
at http://nmt-benchmark.net/. We use Moses tokenizer and do not perform any post-processing,
before computing BLEU score on detokenized sentences (\textit{multi\_bleu.perl} script).

\begin{table}[t]
    \caption{Detokenized BLEU score on newstest2014 data using
     regularized and unregularized attention.
    }\label{table:bleu}
\vskip 0.15in
\begin{center}
\begin{small}
\begin{tabular}{lcc}
    \toprule
    Attention model & WMT14 1M fr$\to$en & WMT14 en$\to$fr \\
    \midrule
    Softmax & \textbf{27.96} & \textbf{28.08} \\
    Entropy regularization &  \textbf{27.96} & 27.98 \\
    $\ell_2^2$ reg. & 27.21 & 27.28 \\
    \bottomrule
    \end{tabular}
\end{small}
\end{center}
\end{table}

\paragraph{Implementation.}We implemeted a batch version of the $\nabla\viterbiOmega$ layer on GPU,
using the \textit{PyTorch} tensor API. Model with negentropy-regularized attention mechanism
runs $1/2$ as fast as the softmax attention mechanism (approximately $7500$ tokens/s vs $15000$ tokens/s
on a single Nvidia Titan X Pascal).
With $\ell_2^2$ regularization, it is only $1/3$ as fast: approximately $5000$ tokens/s. Although this remains reasonable,
it could certainly
be optimized by rewriting kernels using lower-level languages (\textit{e.g.,} using \textit{ATen} API
from \textit{PyTorch}.)

\paragraph{Further results.}Table~\ref{table:bleu} provides BLEU scores for
both translation directions on the 1 million sentence subset of WMT14 we used. We observe
that the introduction of structure and sparsity does not hinder the general performance of the model.
We provide several examples of attention maps in Figure~\ref{fig:plenty_attention},
that illustrate the sparsity patterns $\ell_2^2$ regularization uncovers.

\begin{figure}
    \begin{center}
    \makebox[\textwidth][c]{    
    \includegraphics[width=1.1\textwidth]{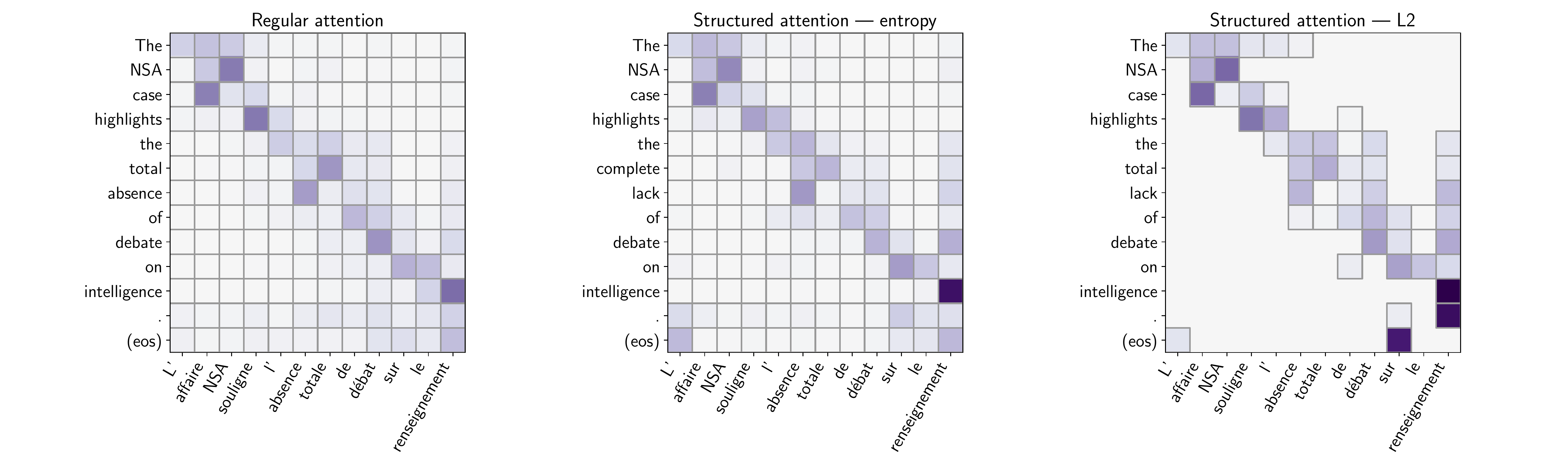}
    }
    \makebox[\textwidth][c]{        
    \includegraphics[width=1.1\textwidth]{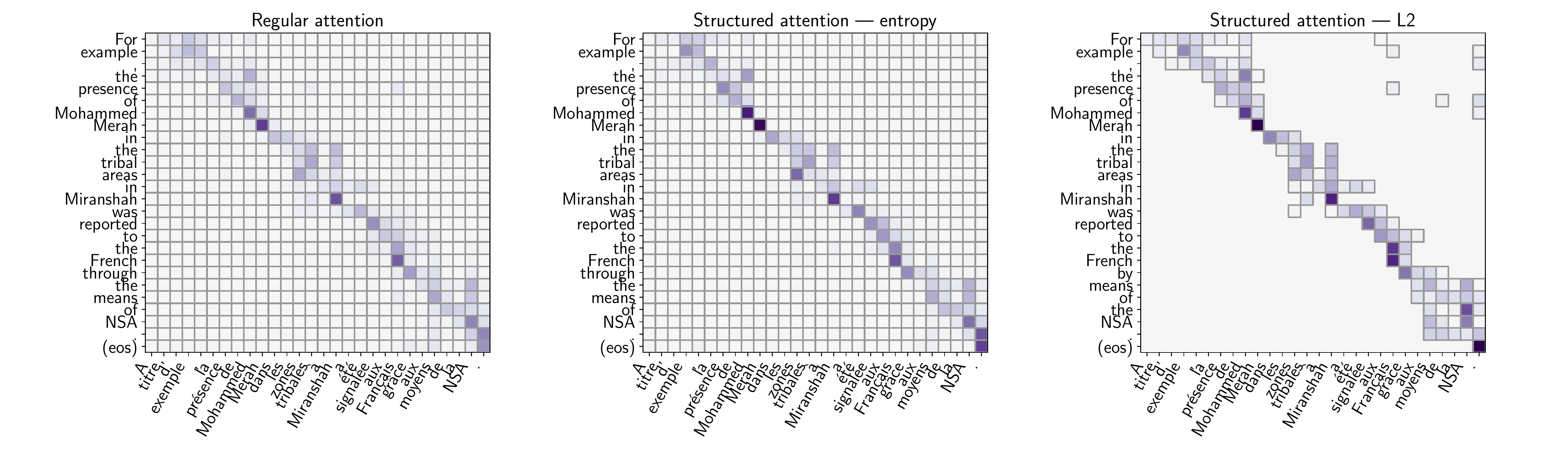}
    }
    \makebox[\textwidth][c]{        
    \includegraphics[width=1.1\textwidth]{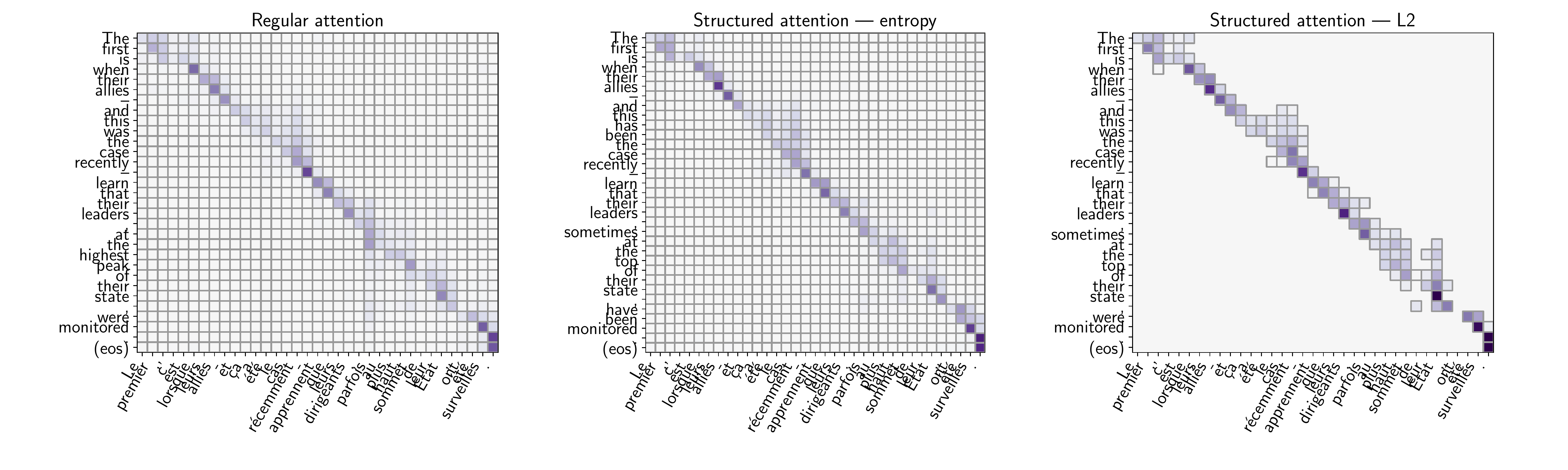}
    }
    \makebox[\textwidth][c]{    
    \includegraphics[width=1.1\textwidth]{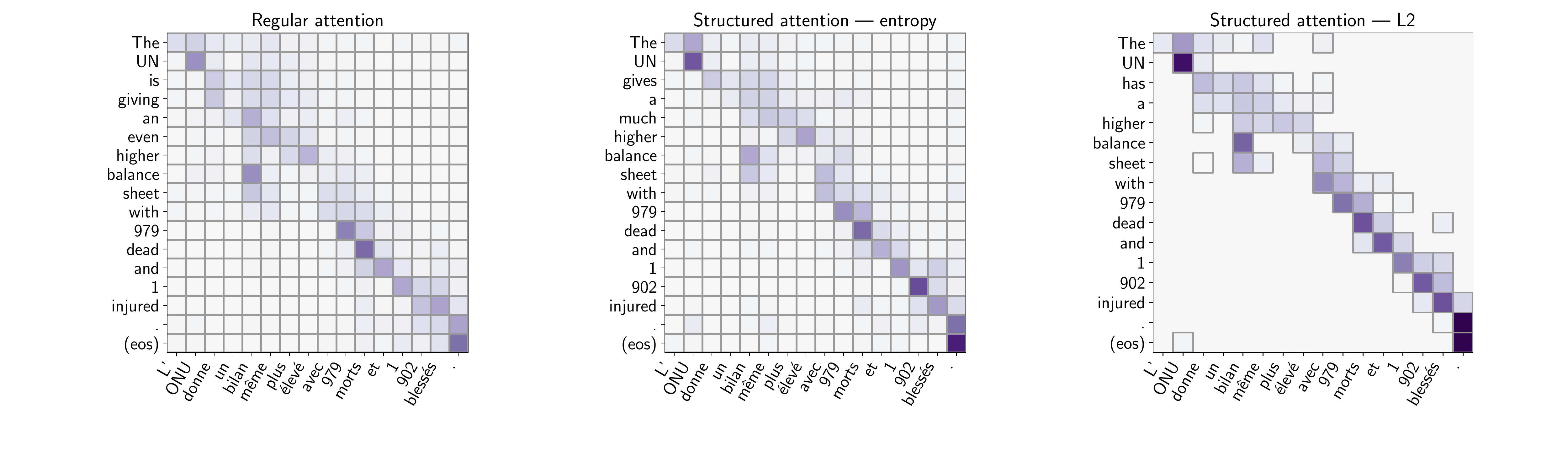}
    }

    \caption{Attention on test samples from Newstest2014. Borders indicate
    non-zero cells. Translations ($y$-axis) are often qualitatively equivalent,
     while attentions maps are sparse 
    in the $\ell_2^2$ case.}\label{fig:plenty_attention}
    \end{center}
\end{figure}

\end{document}